\documentclass[lettersize,journal]{IEEEtran}
\usepackage{amsmath,amsfonts}
\usepackage{array}
\usepackage[caption=false,font=normalsize,labelfont=sf,textfont=sf]{subfig}
\usepackage{textcomp}
\usepackage{stfloats}
\usepackage{url}
\usepackage{verbatim}
\usepackage{graphicx}
\usepackage{cite}
\usepackage{booktabs}
\newtheorem{assumption}{Assumption}
\newtheorem{proposition}{Proposition}

\newtheorem{proof}{Proof}
\usepackage{amsmath,amsfonts}
\usepackage[ruled,vlined]{algorithm2e}
\usepackage{amssymb}

\newcommand{\E}{\mathbb{E}}

\newcommand{\B}{\mathcal{B}}

\begin{document}

\title{Improving Online Reinforcement Learning via Bidirectional Behavior Prior Distillation}

\author{Gong Gao$^{1}$, Xiao Lai$^{2}$, Jiaji Shen$^{1}$, Ning Jia$^{1}$, Xianhui Liu$^{1}$, and Weidong Zhao$^{1}$,
\thanks{$^{1}$ School of Computer Science, Tongji University, China,$^{2}$ Shanghai University of Engineering Science, China 
{\tt\small g18438613630@126.com},
{\tt\small 2210624@tongji.edu.cn},
{\tt\small  laixiao@sues.edu.cn}
{\tt\small  7072@tongji.edu.cn}
{\tt\small xianhui\_l@163.com},
{\tt\small zwd\_tj@tongji.edu.cn},

}%
}

\maketitle

\begin{abstract}
Online reinforcement learning (RL) algorithms frequently exhibit poor sample efficiency and unstable learning dynamics, stemming from systematic critic estimation errors that are exacerbated by greedy policy updates.
Existing behavior-prior reinforcement learning methods attempt to alleviate this issue by relying on offline pre-training to learn behavior models from fixed datasets and using policy priors to constrain online policy updates. However, the limited quality of offline datasets often hinders the ability to provide high-value policies that can effectively guide policy updates. The absence of expert trajectories significantly impairs online policy learning, leading to low sample efficiency and suboptimal performance. To address these challenges, we depart from conventional behavior prior approaches and propose a Bidirectional Behavior Prior Distillation (B2PD) algorithm.
B2PD leverages action-value priors to guide a conditional variational autoencoder (CVAE) in generating a high-value behavior support set. The resulting expert behavior priors are further distilled into the agent, effectively reducing inefficient exploration and enabling stable policy optimization, while establishing a bidirectional knowledge flow mechanism.
Empirical evaluations on both state- and pixel-based tasks verify that B2PD substantially improves sample efficiency while maintaining stable policy optimization.
More broadly, this work shows that enforcing high-quality behavioral support during online learning effectively mitigates critic-induced error amplification, enabling structured behavior priors to guide policy updates in a principled and sample-efficient manner.
\end{abstract}

\begin{IEEEkeywords}
Online reinforcement learning, behavior prior distillation, bidirectional knowledge flow, sample efficiency
\end{IEEEkeywords}
\section{Introduction}
\label{Introduction}
\IEEEPARstart{O}{nline} reinforcement learning (RL) is notoriously sample-inefficient, particularly when compared with offline paradigms such as imitation learning~\cite{elhadad2026grail,xu2026understanding}. This inefficiency is largely attributed to the estimation errors in value function approximation, which are aggravated by greedy policy updates and lead to overestimation~\cite{quangaugmenting}. This feedback loop leads to non-stationary value overestimation~\cite{quangaugmenting}, resulting in unstable learning dynamics and inefficient exploration during policy optimization. To address these challenges, behavior prior reinforcement learning (BPRL)~\cite{singh2020parrot,tirumala2022behavior,daoudi2024enhancing,han2024lifelike} has emerged as a promising research direction. As a generalization of standard online RL, BPRL first pretrains a behavior cloning model on offline expert datasets~\cite{guo2024blend,sun2026prior}, which serves as a teacher policy for policy distillation during online learning. By leveraging behavior priors to guide policy updates, BPRL promotes more stable and sample-efficient policy improvement.

In existing BPRL algorithms, offline pretraining plays a crucial role in accelerating policy convergence. Specifically, it utilizes a behavior cloning model to provide policy priors, which are then distilled into the agent to improve sample efficiency~\cite{ball2023efficient,wagenmaker2023leveraging} and expedite policy convergence.
Although offline pretraining has achieved considerable success within the framework of BPRL, it remains subject to a critical limitation: The generated actions are constrained to samples drawn from static offline datasets, thereby restricting policy quality to that of the pre-collected trajectories~\cite{hao2023leveraging,chemingui2025constraint,jiang2026beyond}. Although some researchers have employed advantage-weighted learning to mitigate the impact of low-quality samples~\cite{chen2022lapo,qing2024a2po}, these approaches still fundamentally rely on the fixed support of offline data.

Moving beyond methods that rely on pre-collected expert datasets, recent dynamic behavior-prior approaches~\cite{shen2021theoretically,daoudi2024enhancing} construct policy priors online by maintaining a policy support set, eliminating the need for offline demonstrations. However, these methods are inherently constrained by the quality of the policy support set, which is composed solely of previously observed actions and therefore lacks the ability to generate unseen high-value behaviors. Consequently, the resulting policy priors are often suboptimal, limiting their effectiveness in guiding policy optimization.
This limitation gives rise to two fundamental challenges. First, the restricted policy support set prevents the generation of high-value action priors that could facilitate stable and efficient policy improvement. Second, suboptimal policy priors provide insufficient guidance for actor optimization, reducing their ability to correct inaccurate policy updates induced by stochastic exploration.

To overcome these limitations, we propose an online policy distillation framework that leverages conditional generative modeling to proactively generate policy priors. Unlike retrieval-based methods, the proposed approach is capable of synthesizing novel, high-quality action priors beyond the observed policy support set. Although variational autoencoders have previously been employed to reconstruct offline behaviors for policy transfer~\cite{yang2025dynamic}, their potential for actively improving exploration and policy optimization in online RL has received comparatively limited attention.

In this work, we seek to answer a fundamental question: Can existing knowledge distillation methods be directly applied to online RL training to incorporate policy priors, without relying on offline pretraining or explicit historical trajectory constraints~\cite{goyal2022retrieval,ran2023policy,li2023accelerating,lyu2025cross}, thereby enabling stable policy improvement?
To explore this, we proactively generate behavioral policy priors and introduce a loss function guided by action values to train a behavior cloning model that improves policy quality. Building on this, we distill expert-level behavior priors into the reinforcement learning agent to enable more efficient policy improvement, as illustrated in Fig.~\ref{overall_architecture}. Empirically, Bidirectional Behavior Prior Distillation (B2PD) delivers substantial gains on standard online RL benchmarks, highlighting its practical effectiveness. Moreover, B2PD is broadly compatible with existing online RL algorithms and can be integrated with methods such as SAC~\cite{Haarnoja2018SAC} and TD3~\cite{fujimoto2018addressing} with only minimal modifications.
\begin{figure}[!ht] 
\centering
\begin{center}
\centerline{\includegraphics[width=0.5\textwidth]{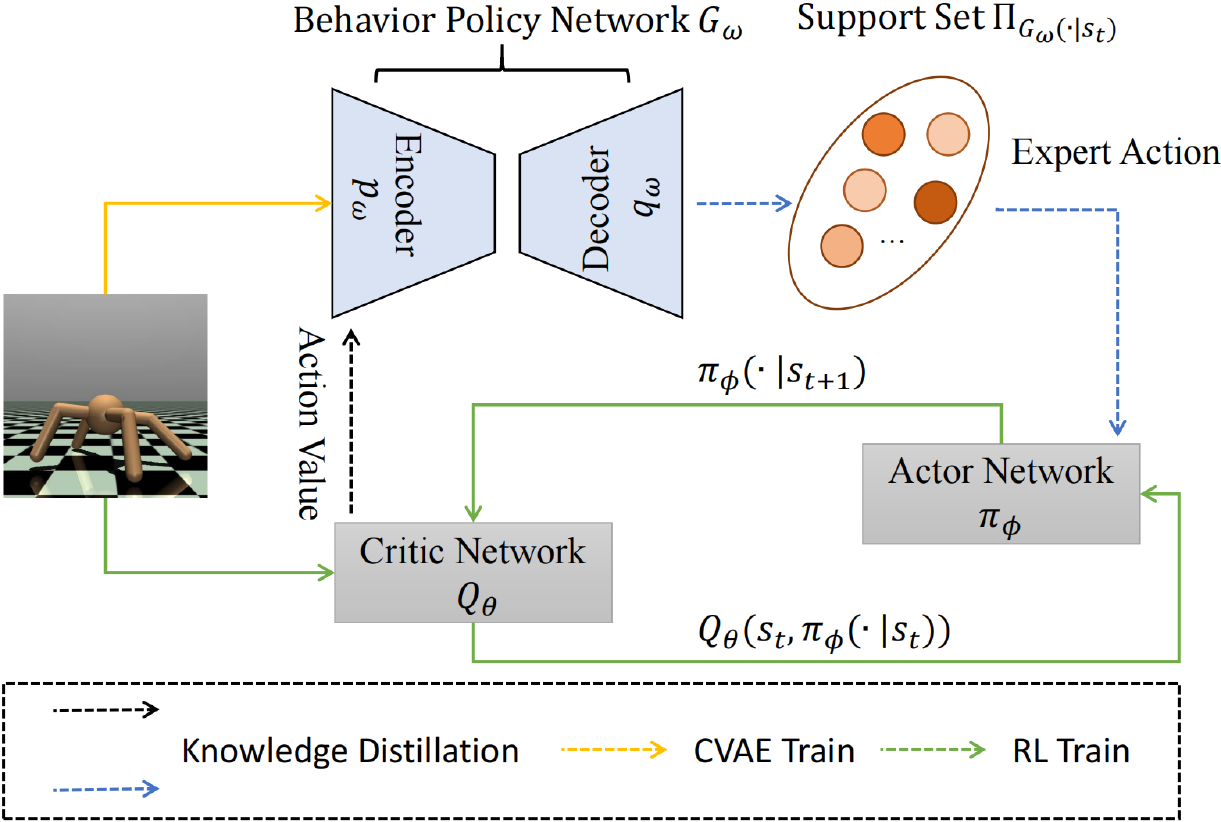}}
\caption{We introduce prior knowledge distillation to enhance learning efficiency and propose bidirectional distillation mechanisms: integrating action value priors into the conditional variational autoencoder (CVAE) network training, and incorporating behavior policy priors during policy updates.}
\label{overall_architecture}
\end{center}
\end{figure}

The contributions of this work are fourfold.
\begin{itemize}
 \item First, we propose B2PD, a novel framework that proactively generates policy priors to facilitate policy optimization in online reinforcement learning. Unlike existing behavior-prior methods that rely on fixed offline datasets or historical trajectories, B2PD dynamically constructs informative policy priors during online interaction.

\item Second, we introduce Action Value Prior Distillation (AVPD), which employs a CVAE to model the policy support set and synthesize high-quality policy priors under Q-value guidance. Building upon AVPD, we further develop Behavior Policy Prior Distillation (BPPD), which distills the generated behavioral priors into the policy through supervised learning, effectively improving exploration efficiency and policy optimization.

\item Third, we propose a Standard Deviation-Aware Noise (SDA Noise) scheduling mechanism that adaptively controls action perturbations to stabilize Q-value estimation.

\item Finally, extensive experiments on both state-based and pixel-based continuous control benchmarks demonstrate that B2PD consistently improves sample efficiency, exploration efficiency, and training stability.
 
\end{itemize}
We begin by introducing the necessary background in Section~\ref{sec:pre} and reviewing related work in Section~\ref{sec:related_works}. Section~\ref{sec:method} presents the proposed B2PD framework, including its algorithmic design and convergence analysis. Experimental evaluations on three continuous control benchmarks are reported in Section~\ref{sec:experiment}. Section~\ref{sec:limitations} discusses limitations and implications of the proposed approach. We conclude in Section~\ref{sec:conclusion_future_works} with a summary and directions for future research.
\section{Preliminaries}
\label{sec:pre}
We begin by establishing the formal notation and background for RL. Subsequently, we detail the maximum entropy RL framework, upon which our method is built as the foundational baseline.
\subsection{Reinforcement Learning}
Within the standard framework of the Markov decision process (MDP), RL can be formulated as $\mathcal{M}=\langle\mathcal{S}, \mathcal{A}, \mathcal{P}, \mathcal{R}, \gamma\rangle$. Here, ${\cal S}$ denotes the state space, ${\cal A}$ denotes the action space, $r: {\cal S\times A}\in [-\mathcal{R}_{max}, \mathcal{R}_{max}]$ denotes the reward function, and $\gamma\in (0,1)$ is the discount factor, and $\mathcal{P}(\cdot \mid s_t,a_t)$ stands for transition dynamics. For simplicity, we denote the current and next state-action pairs as $(s_t, a_t)$ and $(s_{t+1}, a_{t+1})$, respectively. The agent's behavior is defined by a stochastic policy $\pi(\cdot\mid s_t)$, which maps a given state to a probability distribution over possible actions. The state and state-action distributions induced by $\pi$ are denoted by $\rho_{\pi}(s_t)$ and $\rho_{\pi}(s_t, a_t)$, respectively.

\subsection{Maximum Entropy RL}
\label{sec:max_entropy}
In our work, we adopt an entropy-augmented objective~\cite{Haarnoja2018SAC}, which incorporates policy entropy into the reward signal to define the policy optimization objective as
\begin{equation}
\label{eq.policy_objective}
J_{\text{MaxEnt}}(\phi) = \sum_{t=0}^{T}\E_{\substack{(s_t, a_t)\sim  \rho_\pi}} \Big [r(s_t,a_t)+\alpha\mathcal{H}\left(\pi(\cdot \mid s_t)\right)\Big],
\end{equation}
where $\alpha$ is the temperature coefficient, and the policy entropy $\mathcal{H}$ is expressed as
\begin{equation}
\begin{aligned}
\mathcal{H}(\pi(\cdot \mid s_t))=\E_{a_t\sim\pi_\phi}\big[-\log\pi(a_t|s_t)\big].
\end{aligned}
\end{equation}

The optimal policy can be obtained via a maximum entropy variant of policy iteration, which comprises two alternating steps: (a) soft policy evaluation and (b) soft policy improvement. This procedure is collectively referred to as soft policy iteration. Given a policy $\pi$, the corresponding soft Q-value can be learned by iteratively applying a modified Bellman operator $\mathcal{T}_c^\pi$ that incorporates the entropy term $\mathcal{H}$, which can be formulated as
\begin{equation}
\label{eq.soft_bellman_1}
\begin{aligned}
\mathcal{T}_c^\pi Q(s_t,a_t)=&r(s_t,a_t)+\gamma \mathbb{E}_{s_{t+1}\sim \rho_\pi,a_{t+1}\sim \pi_{\phi}} [Q_{\theta ^{'} }(s_{t+1},a_{t+1}) \\
-&\alpha \log\pi_{\phi} (\cdot \mid s_{t+1})]\big],
\end{aligned}
\end{equation}
where $\gamma \in [0,1)$ is the discount factor, $\theta$ denotes the parameters of the Q-network, and $\phi$ denotes the parameters of the Gaussian policy.
The parameters of the soft Q-function are optimized by minimizing the soft Bellman residual, which can be formalized as
\begin{align}
\label{eq:critic_origin}
\mathcal{L}_Q(\theta) = \E_{(s_t, a_t)\sim  \B}{\left(Q(s_t, a_t) - \mathcal{T}_c^\pi Q(s_t,a_t)\right)^2},
\end{align}
where $\B$ denotes a mini-batch sampled from the replay buffer.
\section{Related Works}

\label{sec:related_works}
\textit{Online RL with Static Policy Prior:}
In contrast to training agents directly from scratch, RL with static policy prior~\cite{ball2023efficient,zang2023behavior,spigler2024proximal,choi2025dynamic} leverages explicit knowledge obtained from offline datasets to guide online agent training, thereby accelerating policy convergence. In practical applications, these additional priors are typically used to direct the agent's focus on high-value areas of the Markov Decision Process (MDP), reducing unnecessary exploration.

Initially, these data are employed to initialize policies through behavior cloning (BC)~\cite{george2023one}, and they permeate throughout the optimization process. For instance, in RLPD~\cite{ball2023efficient} and EXPLORE~\cite{li2023accelerating}, demonstration data are incorporated into the replay buffer, and new regularization terms are introduced to force the optimization process to better utilize the demonstration data. Similarly, POfD~\cite{kang2018policy}, LOGO~\cite{rengarajan2022reinforcement}, BPR~\cite{zang2023behavior}, PPD~\cite{spigler2024proximal}, DCSL~\cite{choi2025dynamic} and DICE-RL~\cite{sunprior} impose penalties or constraints on the RL objective, forcing the agent's policy to remain close to the policy prior.
However, these methods are only feasible when expert policies are available, as they cannot effectively guide policy updates in states with low-quality interactions~\cite{qing2024a2po,beliaev2025inverse}. Furthermore, these approaches might require substantial data to direct the RL agent, and acquiring such data can be challenging in systems with high sampling costs.

\textit{Online RL with Dynamic Policy Prior:}
Dynamic policy prior methods offer a flexible framework that bypasses the need for static offline datasets. These approaches typically facilitate exploration by leveraging local policy guidance, either by integrating historical trajectories to enrich representation~\cite{kapturowski2018recurrent,quangaugmenting} or by constraining the action space to stabilize learning~\cite{shen2021theoretically,daoudi2024enhancing}.

One prominent strategy involves exploiting historical interaction data. R2D2~\cite{kapturowski2018recurrent} employs recurrent state representations to capture temporal dependencies, while ALH~\cite{quangaugmenting} utilizes past observations to provide action guidance. However, such methods are prone to error propagation during early training stages due to high epistemic uncertainty. To improve sample utilization, recent studies~\cite{ji2023seizing,luo2024offline} have incorporated historically high-value actions into the value estimation process. Specifically, BAC~\cite{ji2023seizing} mitigates Q-value underestimation by shifting from action-value estimation to state-value estimation, while OBAC~\cite{luo2024offline} employs a concurrently learned offline policy as an adaptive regularizer. Nevertheless, their reliance on historical behaviors gradually increases throughout training, which may bias policy optimization toward suboptimal behaviors and reduce exploration diversity, especially in environments with complex and multimodal reward landscapes.

Alternatively, nearest-neighbor RL algorithms implement constrained exploration to alleviate bootstrapping errors~\cite{shah2018q}. Theoretical work by NNAC~\cite{shen2021theoretically} demonstrates that high-quality actions retrieved from local neighborhoods serve as effective policy anchors. Similarly, IRA~\cite{gao2026improving} utilizes retrieved high-quality trajectories to enhance sample efficiency, and RLLG~\cite{daoudi2024enhancing} employs approximate policy evaluation to direct local guide policies toward better actions.

Our approach aligns with existing methods in leveraging local policy guidance~\cite{shen2021theoretically,daoudi2024enhancing,gao2026improving}, yet fundamentally departs from retrieval-based frameworks due to its generative nature. By integrating action-value prior distillation, our framework provides a continuous and expressive guidance signal that regularizes actor updates. This mechanism not only mitigates the inefficient exploration inherent in maximum-entropy RL, but also robustly corrects suboptimal gradients, thereby promoting stable policy optimization without relying on offline datasets.

\section{Methods}
\label{sec:method}
In this section, we first investigate why the policy support set naturally contains high-quality behavioral priors and how these priors can be exploited to improve policy optimization in online RL. Based on this insight, we develop the proposed B2PD framework on top of SAC, as summarized in Algorithm~\ref{alg:B2PD_online}. First, AVPD employs a CVAE to learn the policy support set from the online replay buffer and synthesize high-quality policy priors under Q-value guidance. Second, BPPD transfers the generated behavioral priors to the policy network through supervised learning, where high-value actions are selected as reliable update anchors for policy optimization. Finally, an SDA Noise scheduling mechanism adaptively controls action perturbations to stabilize Q-value estimation, providing reliable supervision for AVPD and enabling the generation of more informative policy priors.

\subsection{Existence of Expert Behavior Prior}
We first introduce two standard assumptions used throughout the theoretical analysis: the approximation assumption and the support coverage assumption.

\begin{assumption}[Universal Multimodal Policy Approximation~\cite{huang2023reparameterized}]
\label{assump:cvae_ua_1}
Let CVAE ${G}_\omega$ denote the class of behavior policies parameterized by $\omega$, where the latent variable $z \sim p(\cdot \mid s_t)$ captures context-dependent stochasticity and the decoder $q(\cdot \mid s_t, z)$ maps to actions. Then, under sufficient model capacity and data coverage, for any continuous and potentially multimodal optimal policy, there exists $G_\omega(\cdot \mid s_t)$ that approximates the actor $\pi_{\phi}(\cdot \mid s_t)$.
\end{assumption}

\begin{assumption}[Policy Support Set Inclusion~\cite{achiam2017constrained}]
\label{prop:support_cvae_1}
Let $\Pi_{G_{\omega}}$ denote the set of actions with non-zero density under a CVAE-based policy support set $\Pi_{{G_{\omega}}(\cdot \mid s_t)}$, and ${\pi_{\phi}(\cdot \mid s_t)}$ denote a Gaussian policy. Then, for any $s_t \in \mathcal{S}$,
${\pi_{\phi}(\cdot \mid s_t)} \subsetneqq  \Pi_{{G_{\omega}}(\cdot \mid s_t)}$,
where the inclusion is strict if $\Pi_{G_{\omega}}$ approximates a multimodal policy distribution.
\end{assumption}

 \begin{proposition} Suppose $s_t$, $a_t$, and $r_t$ follow MDP model. Under the assumption that the generative model is universal policy approximation and multimodal, there exists at least one action $\tilde{a} \in  \Pi_{G_{\omega}}(\cdot |s_t)$ such that $Q_{\theta}(s_t, \tilde{a})> Q_{\theta}(s_t, \pi(\cdot \mid s_t))$.
\label{prop:multi_model}
\end{proposition}

\begin{proof}
We first prove that there exists at least one action prior $\tilde{a} \in \Pi_{G(\cdot |s_t)}$ such that 
$
Q_{\theta}(s_t, \tilde{a}) > Q_{\theta}(s_t, \pi^k(\cdot \mid s_t)),
$
where $\pi^{k}$ denotes the policy at RL training iteration $k$.
We prove this by contradiction. Suppose for all $s_t,\tilde a$ that 
$
Q_{\theta}(s_t, \tilde{a}) < Q_{\theta}(s_t, \pi^k(\cdot \mid s_t)).
$
Then, we have 
$
Q_{\theta}(s_t, \pi^k(\cdot \mid s_t)) = Q_{\theta}(s_t, \pi^{\star}(\cdot \mid s_t)).
$
However, this contradicts the statement 
$
Q_{\theta}(s_t, \pi^k(\cdot \mid s_t)) < Q_{\theta}(s_t, \pi^{k+1}(\cdot \mid s_t)).
$
Thus, this contradicts our initial claim regarding the existence of the expert action prior. We contend that within a finite number of training steps, there exists at least one action within the behavior policy support set of the generative network that achieves a Q-value greater than that of the action selected by the current policy.

\end{proof}

\subsection{Action Value Prior Distillation}
In this section, we provide a comprehensive overview of the construction of the behavior policy model. Considering the mode collapse issue commonly observed in GANs~\cite{liu2019spectral}, we avoid using GAN-based architectures to learn the policy prior. Consequently, we adopt a CVAE to build the behavior policy model $G$. During action reconstruction, the model incorporates action value estimates as prior knowledge to guide $G$ toward producing higher-value actions.

The CVAE $G_\omega$, parameterized by $\omega$, consists of an encoder $p_\omega(s_t, a_t)$ and a decoder $q_\omega(s_t,z)$, denoted as $G = \{p, q\}$. The CVAE is optimized by maximizing its evidence lower bound objective (ELBO), which can be formulated as
\begin{equation}
    \label{eq:Rec}
    \begin{aligned}
    \mathcal{V}_{\text{ELBO}}(\omega) = &\mathbb{E}_{\substack{(s_t, a_t)\sim \B},z\sim p_\omega(s_t, a_t)}\big[ \left(a_t - q_\omega(s_t,z)\right)^2    \\ +&{\text{KL}}\left(p_\omega(s_t, a_t)|| \mathcal{N}(0,{\bf I})\right) \big],
    \end{aligned}
\end{equation}
where ${\text{KL}}(p||q)$ denotes the kullback–leibler  (KL)~\cite{hinton2015distilling} divergence between the probability distributions $p(\cdot)$ and $q(\cdot)$, and $\mathbf{I}$ represents the identity matrix. When sampling actions from the CVAE, we first sample a latent variable $z$ from the prior distribution, which is assumed to follow a normal distribution $\mathcal{N}(0, \mathbf{I})$. This latent variable, along with the state $s_t$, is then passed into the decoder $q_\omega(s_t,z)$ to obtain the decoded action.

Furthermore, we expect the value estimation prior in the Q-network to serve as an effective guidance signal, enabling the CVAE to generate actions with higher expected returns.
As noted by DIDI~\cite{liu2024didi}, employing Q-value guided gradients to train a generative model yields diverse and high-value actions. We propose an action-value prior distillation loss, which leverages double Q-networks to guide the gradient update direction of the CVAE. The Q-value prior distillation loss can be formulated as
\begin{equation}
    \label{eq:cvae_q}
     \mathcal{V}_{\text{Dist}}(\omega) = -  \mathbb{E}_{(s_t, a_t)\sim\mathcal{\B},z\sim p_\omega(s_t, a_t)}    {Q_\theta{(s_t,}q_\omega(s_t,z))}  .
\end{equation}

We optimize the complete CVAE model by integrating the two aforementioned loss components. To this end, we introduce a weighting parameter $\xi$ to balance the trade-off between action reconstruction and the value of the generated actions. The overall loss function is then formulated as
\begin{equation}
    \label{eq:CVAE}  
    \mathcal{V}_{G}(\omega) = \mathcal{V}_{\text{ELBO}}(\omega) +\xi \mathcal{V}_{\text{Dist}}(\omega).
\end{equation}

\subsection{Behavior Policy Prior Distillation}
\label{sec:hpg}
We leverage the behavior policy model $G$ to generate the policy prior $G_\omega(s)$, which provides a high-value anchor for policy updates during the actor optimization.
The behavior prior is obtained via multiple forward passes of the generative network derived from the policy support set, where the anchor ${\tilde {a}}$ for policy updates is chosen based on Q-target evaluations, and is formulated as
\begin{equation}
\label{eq:optimal_a}
{\tilde a} = {\arg{\mathop{\max \limits_{a_h}}}} \left(Q_{\theta^{'}}(s_t,a_h)\right), a_h \in \Pi_{G_{\omega}(\cdot \mid s_t)},
\end{equation}
where $\arg\mathop{\max}\limits_{a_h}$ denotes the process that identifies optimal actions from the policy prior set $\Pi_{G_{\omega}(\cdot \mid s_t)}$, and $h = {1, 2, 3, \dots, H}$ denotes the indices of $H$ prior actions generated by the CVAE.
Following the identification of the optimal action, we implement a systematic optimization mechanism that progressively refines the policy network with the optimal actions through a knowledge distillation paradigm consistent with the algorithm established in~\cite{guo2024blend}.
We incorporate expert prior knowledge into the actor network by minimizing the KL divergence as the training objective, which can be formalized as
\begin{equation}
\label{eq:heu}
\begin{aligned}
J_{\text{Dist}}(\phi)=-{\mathbb{E}_{s_t  \sim \B }}\left[ \text{KL}{{(\pi_\phi(\cdot \mid  s_t) || {{\tilde a}})}} \right].
\end{aligned}
\end{equation}

Notably, actions sampled from the policy support set are not guaranteed to be strictly superior to the current policy action. To address this issue, we introduce a value-aware and dynamically weighted distillation scheme, which selectively applies prior distillation only to high-quality sampled actions and adaptively controls the distillation strength.
Specifically, the distillation weight $\zeta$ is defined as
\begin{equation}
\label{overall_eta_weight}
\zeta
= \eta \frac{
\max\bigl(Q_{\theta}(s_t,\tilde a) - Q_{\theta}(s_t,\pi_{\phi}(\cdot \mid s_t)), 0\bigr)
}{
\mathbb{E}_{s_t \in \mathcal{B}}
\left[
\max\bigl(Q_{\theta}(s_t,\tilde a) - Q_{\theta}(s_t,\pi_{\phi}(\cdot \mid s_t)), 0\bigr)
\right]+c
},
\end{equation}
where $\eta$ is a hyperparameter regulating the overall distillation intensity, $\max(\cdot, \cdot)$ represents the rectified value advantage, and $c$ is a small constant added for numerical stability.

Then, we integrate the maximum entropy objective with the expert policy prior distillation loss.
The overall loss function of the Actor network can be formulated as
\begin{equation}
\label{overall_actor}
J_{\pi}(\phi)=   J_{\text{MaxEnt}}(\phi)+ \zeta   J_{\text{Dist}}(\phi) ,
\end{equation}
where $\zeta$ denotes the weight of the policy distillation loss $J_{\text{Dist}}$.

\subsection{Standard Deviation-Aware Noise Scheduling}
 To achieve stable Q-value estimation, we propose the SDA Noise scheduling mechanism, which dynamically adjusts the noise magnitude based on the predicted action distribution. This further refines the Q-target value estimation process in the vanilla SAC~\cite{Haarnoja2018SAC} algorithm.   

\begin{proposition}[Noise-Regularized Q-Value Smoothness]
\label{assump:noisy_Q_smooth}
The soft Q-value ${Q}(s_t, a_t)$
is assumed to be Lipschitz continuous~\cite{geist2019theory}. Moreover, the Q-function $Q(s_t, a_t)$ is twice differentiable concerning $a_t$, and there exists two constants $C_1 > 0$, $C_2 > 0$ such that $\left\| \nabla_a Q(s_t, \hat a) \right\| \leq C_1, \left\| \nabla_a^2 Q(s_t, \hat a) \right\| \leq C_2$. where $\hat{a} \sim \mathcal{U}(a_t, a_t + \epsilon)$ is an action sampled uniformly between $a$ and the noisy action $a + \epsilon$. That is, for all states \( s_t \in \mathcal{S} \) and Gaussian action noise \( \epsilon \sim \mathcal{N}(0, \tau^2\mathbf{I}) \), the following inequality holds
\begin{equation}
\label{eq:smooth}
\E_{\epsilon}\left|  {Q}(s_t, a_t) -   {Q}(s_t, a_t+\epsilon) \right| \leq  C_1\sqrt{\frac{2}{\pi}}\tau+ \frac{1}{2}C_2\tau^2,
\end{equation}
where $\tau$ denotes the standard deviation of the Gaussian noise.
This smoothing effect facilitates more stable Q-value estimation and mitigates high variance in bootstrapped Q-learning.
The proof is presented in Appendix~\ref{thero_ana_b2d}.
\end{proposition}

Under the maximum entropy framework, adding Gaussian noise $\epsilon$ to actions yields smoother and more stable action-value estimates, formulated as
\begin{equation}
\label{target_q_loss}
\begin{aligned}
\mathcal{T}_c^\pi Q_{\theta}(s_t, a_t) =& r + \gamma \, \mathbb{E}_{s_{t+1} \sim \rho_\pi,\, a_{t+1} \sim \pi_{\phi}} \big[ Q_{\theta'}(s_{t+1}, a_{t+1} + \epsilon)\\ - &\alpha \log \pi_{\phi}(a_{t+1} \mid s_{t+1}) \big].
\end{aligned}
\end{equation}

Since action noise in SAC can lead to abrupt changes in policy entropy, we exclude noise-perturbed actions when computing entropy to ensure a more stable and accurate estimate.
To achieve stable Q-value estimation, we propose an SDA Noise scheduling mechanism. Specifically, the action distribution is modeled as a Gaussian policy by the actor network, defined as 
\begin{equation}
    \pi_{\phi}(\cdot \mid s_t) = \mathcal{N}\left(\mu_{\phi}(s_t), \operatorname{diag}(\sigma_{\phi}^2(s_t))\right),
\end{equation}
where $\mu_{\phi}(s_t)$ denotes the mean and $\operatorname{diag}(\sigma_{\phi}^2(s_t))$ represents the diagonal covariance matrix parameterized by the predicted standard deviation vector $\sigma_{\phi}(s_t)$.
We design a noise scheduling function that adaptively scales the action noise, where the noise standard deviation $\tau$ is formulated as
\begin{equation}
    \label{eq:noise_schedule}
    \tau =\frac{1}{d_a}* \sum_{i=1}^{d_a} \frac{0.2}{\exp\left(\sigma^{i}_{\phi}(s_t)\right)^{1.5}},
\end{equation}
where $d_a$ denotes the dimensionality of the action space, and $\sigma^{i}$ is the predicted standard deviation vector of the $i$-th element. This formulation is closely related to the noise injection mechanism in TD3 and adopts the same maximum noise magnitude of 0.2.
This mechanism allows the agent to adaptively adjust the noise standard deviation based on the uncertainty of the predicted action distribution. When the predicted action standard deviation is large, the noise standard deviation $\tau$ correspondingly decreases, resulting in a reduced exploration range that benefits the stability of Q-value estimation. Conversely, a decrease in the action standard deviation expands the exploration range, facilitating broader exploration of the action space.

\begin{algorithm}[!ht]
\setlength{\algomargin}{0.5em}
\SetAlgoInsideSkip{smallskip}
\caption{Bidirectional Behavior Prior Distillation}
\label{alg:B2PD_online}
Initialize behavior policy network $G_{\omega}$, critic networks $Q_{\theta}$, 
and actor network $\pi_\phi$ with random parameters $\omega$, $\theta$, $\phi$\;

Initialize target networks $\theta' \leftarrow \theta$, and initialize empty replay buffer $\mathcal{B}$\;

\For{$t = 0$ \KwTo $T$}{
    Sample action $a \sim \pi_\phi(\cdot \mid s_t)$\;
    
    Get reward $r_t$ and new state $s_{t+1}$, and store transition tuple $(s_t, a_t, r_t, s_{t+1})$ in $\mathcal{B}$\;
    
    Sample mini-batch of transitions $(s_t, a_t, r_t, s_{t+1}) \sim \mathcal{B}$\;
    
Obtain the noise standard deviation $\tau$ via Eq.~\ref{eq:noise_schedule}\;

Obtain the Q-target via Eq.~\ref{target_q_loss}\;
    
    \textbf{Train critic:}   Update critic by minimizing Eq.~\ref{eq:critic_origin}\;
    
    \textbf{Train behavior prior network:}   Update CVAE by minimizing Eq.~\ref{eq:CVAE}\;
    
       Get the expert policy anchor $\tilde a$ by Eq.~\ref{eq:optimal_a}\;
    
    \textbf{Train actor:} Update actor by maximizing Eq.~\ref{overall_actor}\;
    
    \textbf{Update weights:}    $\theta' \leftarrow \lambda \theta + (1 - \lambda)\theta'$ ;
}

\end{algorithm}
\subsection{A Toy Example}
\label{sec:toy_exper}
To provide a more intuitive view of the effectiveness of B2PD compared to maximum entropy reinforcement learning methods, we conduct a comparative analysis of trajectories generated by B2PD and SAC at different training steps. We focus on two evaluation aspects: (1) Can B2PD reduce inefficient random exploration? (2) Does the trajectory distribution cover high-reward regions?

\begin{figure*}
    \centering
\begin{center}
\centerline{\includegraphics[width=1.0\textwidth]{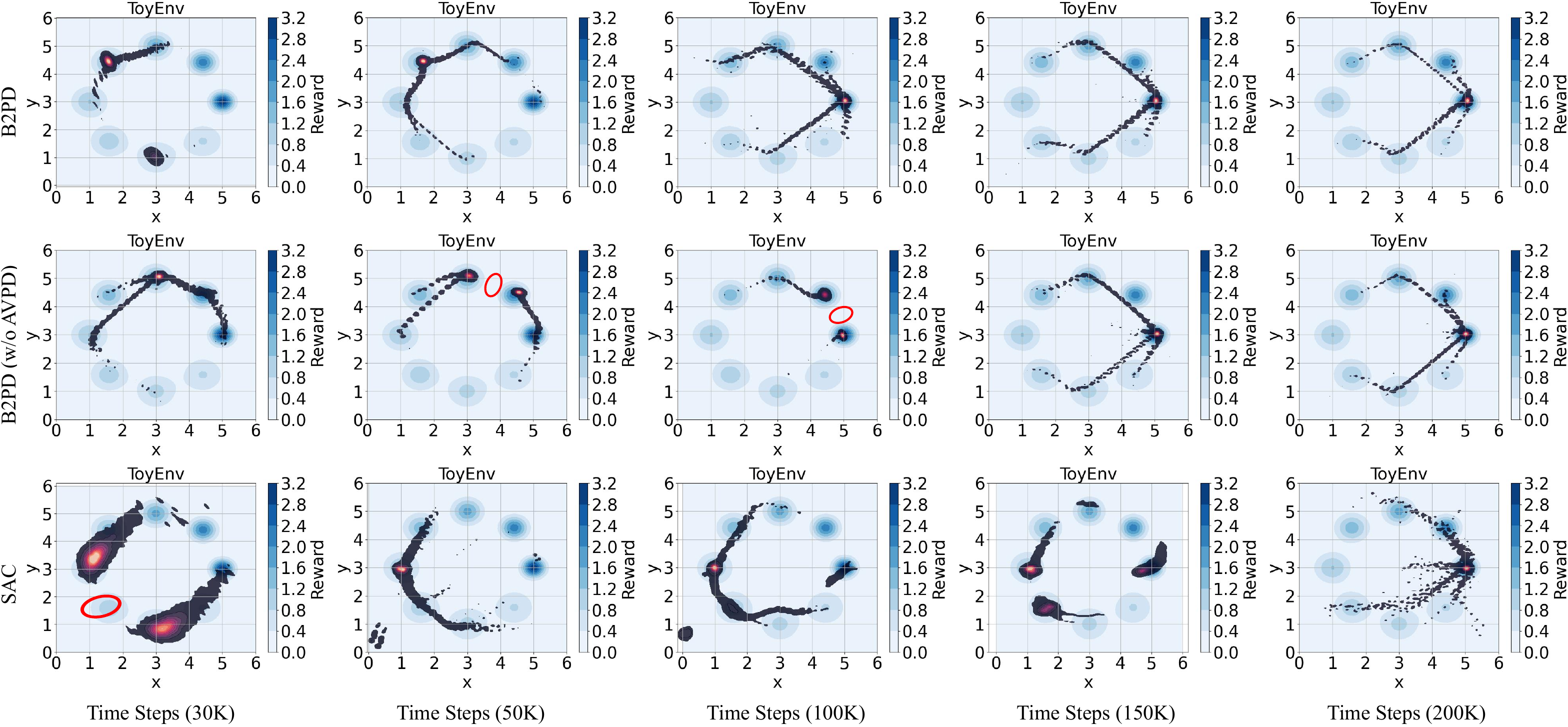}}
\caption{
Visualization of reward and state visitation in the ToyEnv. 
Blue shading indicates the reward magnitude, with darker blue representing higher rewards. 
Red shading corresponds to the state visitation frequency, where more intense red reflects states visited more frequently.
Regions shown in white or pure blue indicate areas that are rarely or never explored.
Each row corresponds to one of the evaluated algorithms (B2PD, B2PD (w/o AVPD), and SAC), and each column corresponds to different training timesteps (30K, 50K, 100K, 150K, and 200K). 
}
\label{fig:Over_toyenv}
\end{center}
\end{figure*}
We construct a toy environment with a state space formed by a mixture of eight Gaussian reward centers, where $S=[S_x, S_y]= [0,6]^2$ denote the state dimensions, the action space is defined as $A=[A_x, A_y] = [-1,1]^2$, and the state transition follows $s_{t+1} = s_t + a_t \times 0.1$.
To examine whether the agent exhibits inefficient exploration during initial training, we estimate the kernel density of state distributions from validation trajectories and visualize the visitation frequency across the state space. 
We visualize the state trajectory distributions at training timesteps 30K, 50K, 100K, 150K, and 200K, using 100 sampled trajectories per checkpoint, as illustrated in Fig.~\ref{fig:Over_toyenv}. 

From these visualizations, we draw three key observations:

First, the results show that B2PD tends to explore nearby high-reward regions progressively. This structured exploration strategy effectively reduces inefficient exploration. In contrast, SAC often falls into local optima even at 150K timesteps, indicating that blind exploration driven by stochastic policy and entropy regularization leads to suboptimal learning efficiency.

Second, compared with SAC, B2PD significantly reduces the visitation frequency of low-reward states, exhibiting superior sample efficiency.
We attribute this to B2PD explicitly distilling the expert prior policy into the agent, such that each state transition is guided by expert action anchors, thereby reducing ineffective exploration.

Third, while both B2PD and its ablated variant B2PD (w/o AVPD) exhibit support-set-based exploration behavior, their exploration patterns differ significantly. As illustrated in Fig.~\ref{fig:Over_toyenv}, B2PD (w/o AVPD) still contains a considerable amount of inefficient random exploration. In particular, at 50K and 100K timesteps, trajectories of B2PD (w/o AVPD) rarely traverse the intermediate region between the two red-shaded areas, indicating poor coverage of valuable transitional states. In contrast, B2PD demonstrates a progressive exploration process throughout training, gradually expanding from low-value regions toward high-value areas. This behavior reflects a structured and value-driven exploration mechanism induced by high-quality action support sets, ultimately facilitating more effective policy improvement.

We argue that the entropy-driven exploration mechanism in SAC expands the action distribution in a largely uniform manner, encouraging broad exploration but ignoring the underlying value structure. Consequently, SAC may spend substantial effort exploring inefficient or redundant regions of the state–action space, particularly when high-reward areas constitute only a small subset, as shown in Fig.~\ref{fig:Over_toyenv}.
In contrast, B2PD restricts exploration toward value-aligned, high-quality modes encoded in the distilled prior. This induces a more structured exploration pattern, where the policy deliberately avoids low-value regions that contradict the prior knowledge, ultimately leading to more efficient exploration and faster convergence.

\section{Experiments}
\label{sec:experiment}
In this section, we evaluate the performance of B2PD using two types of observations.
For state-based inputs, we consider seven MuJoCo continuous-control tasks, four PyBullet tasks, and four DMControl suite tasks.
For visual inputs, we evaluate the method on four pixel-based DMControl tasks.
These experiments aim to answer the following research questions:
(1) Can B2PD leverage behavior policy to accelerate online learning?
(2) Why does exploration guided by a behavior prior yield better performance than entropy-driven exploration?
(3) Does the algorithmic module built upon environment-specific design choices yield consistently reliable performance?
For (1), we first qualitatively demonstrate the effectiveness of B2PD through the toy experiments in Section~\ref{sec:toy_exper}. We then provide a quantitative comparison between B2PD and existing methods that leverage policy priors to accelerate online learning. For (2), we analyze the differences between policy prior distillation and the original algorithm in a toy experiment. Lastly, for (3), we conduct experiments using the B2PD algorithm with a fixed set of parameters across three challenging environments. 
 
\subsection{Three Continuous Control Environments}
We evaluate our method on three widely used continuous control benchmarks: MuJoCo~\cite{todorov2012mujoco}, PyBullet~\cite{coumans2016pybullet}, and DMControl~\cite{tassa2018deepmind}, as shown in Fig.~\ref{task_visual}. We provided a detailed overview of these environments.
\begin{figure*}[htbp]
\begin{center}
\centerline{\includegraphics[width=\textwidth]{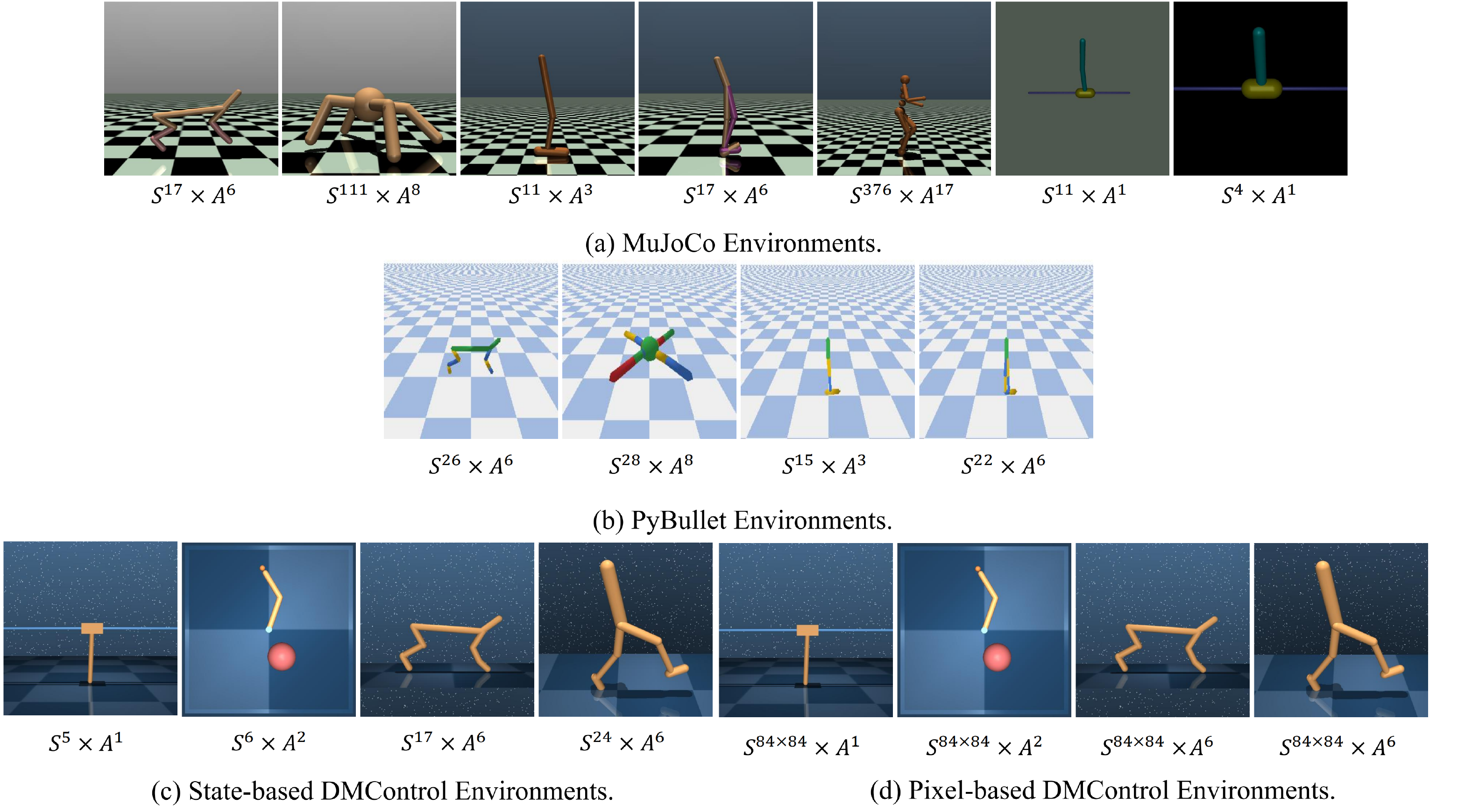}}
\caption{Images of the MuJoCo, PyBullet, and DMControl environments used in our experiments, including 7 MuJoCo continuous control tasks, 4 PyBullet continuous control tasks, 4 state-based DMControl continuous control tasks, and 4 pixel-based DMControl continuous control tasks.}
\label{task_visual}
\end{center}
\end{figure*}

\textbf{MuJoCo.}
MuJoCo is a fast and powerful physics engine that has become a standard platform for simulating complex robotic dynamics and control tasks. It is extensively utilized in the robotics and RL communities for training agents to perform tasks such as locomotion and manipulation, and it offers a diverse set of benchmarks for evaluating RL algorithms. In our experiments, we adopt seven challenging tasks from the MuJoCo environments: HalfCheetah, Ant, Hopper, Walker2d, Humanoid, InvertedDoublePendulum, and InvertedPendulum.

\textbf{PyBullet.}
PyBullet is an open-source physics engine that supports general-purpose physical simulation as well as high-fidelity modeling for robotics tasks. It provides an efficient control interface for fast simulation of rigid-body dynamics, collision detection, and various control problems. Compared to MuJoCo, PyBullet typically exhibits higher levels of environment-induced dynamical noise, making it a more challenging testbed for evaluating the robustness of reinforcement learning algorithms. In this work, we evaluated B2PD on four PyBullet benchmark tasks: HalfCheetahBulletEnv, AntBulletEnv, HopperBulletEnv, and Walker2DBulletEnv.

\textbf{DMControl.}
The DeepMind Control (DMControl) suite is a set of continuous control environments mainly designed for robotic manipulation tasks, but it also covers a wide range of industrial control challenges.
In our experiments, we evaluated the proposed algorithm on four continuous control tasks from the DMControl suite: cartpole-swingup, reacher-easy, cheetah-run, and walker-walk. Each task was assessed under both state-based and pixel-based inputs, with identical action spaces; the only difference lies in the form of the observations.

\subsection{Experimental Setup}

\textbf{Baselines.}
For state-based control settings with low-dimensional physical signal inputs, we compare B2PD against a diverse set of representative off-policy actor-critic baselines, including DDPG~\cite{lillicrap2015continuous}, TD3~\cite{fujimoto2018addressing}, and SAC~\cite{Haarnoja2018SAC}, as well as their enhanced variants that improve exploration (ICM~\cite{pathak2017curiosity}) or policy learning (ALH~\cite{quangaugmenting}, NNAC~\cite{shen2021theoretically}, and BAC~\cite{ji2023seizing}). These methods cover deterministic and stochastic policy optimization, intrinsic exploration, historical state augmentation, nearest-neighbor guided updates, and blended value estimation.

We further instantiate our method on both TD3 and SAC backbones (denoted as B2PD (TD3) and B2PD), where a CVAE-based behavior prior is learned to guide policy updates via prior distillation.

For pixel-based environments, we adopt RAD~\cite{laskin2020reinforcement} and DrQ-v2~\cite{yarats2022mastering} as baselines and integrate our method with both frameworks to assess its effectiveness under visual observations.
To ensure architectural consistency, the CVAE used in B2PD-RAD and B2PD-DrQ-v2 shares the same CNN backbone as the encoder $p_\omega$, while only introducing an additional decoder $q_\omega$ to generate the policy prior.

\textbf{Hyperparameters.}
In Section~\ref{sec:method}, we introduce three parameters: an action value prior coefficient $\xi$=0.1, the number of sampled actions $H$=10, and a behavior prior coefficient $\eta$=1.0.
Unless otherwise stated, we maintain consistent hyperparameter configurations across all experiments. 
Regarding model architecture, our method strictly replicates the structural framework of the SAC algorithm.
All experiments are conducted in a Linux environment equipped with a 56-core Xeon(R) 6133 CPU and 1 Nvidia RTX 3090 GPU.

\textbf{Random seeds.} To ensure the reproducibility of B2PD, we evaluated each algorithm using 10 random seeds. Furthermore, we maintained consistent seeds across all experiments, applying them to PyTorch, Numpy, Gym, and CUDA packages.

\textbf{Evaluation.}
No data or parameters were reused during evaluation, and each evaluation step consisted of 10 randomly initialized episodes. We train for 2 million timesteps on MuJoCo~\cite{todorov2012mujoco} and PyBullet~\cite{coumans2016pybullet} environments, while only 500K timesteps are used for training on the DMControl~\cite{tassa2018deepmind} environments. The policy is evaluated every 5000 timesteps.
We present the mean and 95\% confidence interval over the final 10 trials.

\subsection{Evaluation on Various Benchmark Suites}

\begin{figure*}[!htbp]
\begin{center}
\centerline{\includegraphics[width=\textwidth]{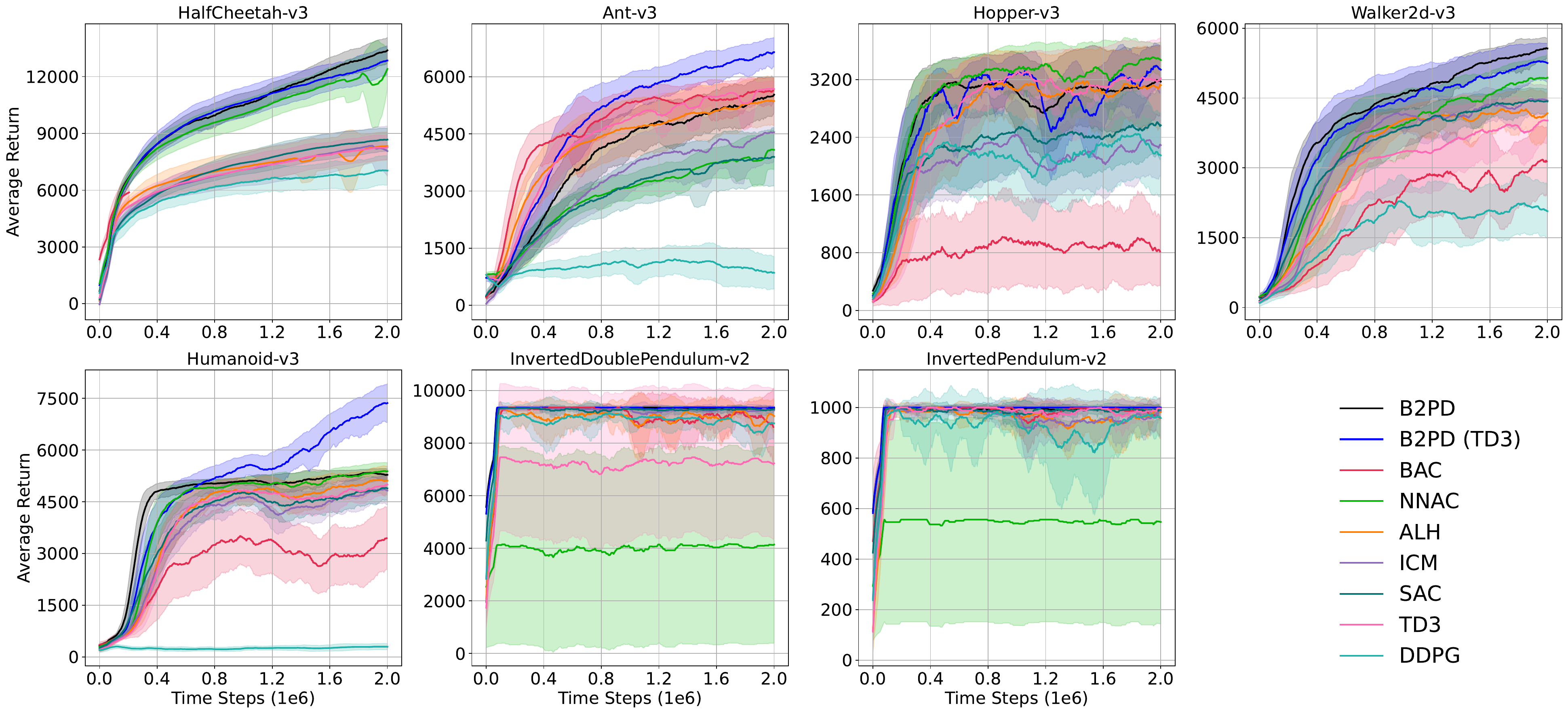}}
\caption{Learning curves on seven MuJoCo continuous control tasks. The shaded area captures a 95\% confidence interval around the average performance over 10 trials. Curves are smoothed uniformly for visual clarity.}
\label{fig:Over_results}
\end{center}
\end{figure*}

\textbf{Main results on MuJoCo environments.}
We compare our proposed B2PD algorithm to these baselines, and these comparisons are performed on seven MuJoCo continuous control tasks. 
Learning curves are displayed in Fig.~\ref{fig:Over_results}, and final results are listed in Table~\ref{tab:mojoco_res}.

 \begin{table*}[!ht]
\centering
\setlength\tabcolsep{3.5pt}
\begin{tabular}{lccccccc|cc}
\toprule
Tasks      &DDPG&TD3&SAC&ICM&ALH&NNAC&BAC&B2PD (TD3)&B2PD   \\  
\midrule
HalfCheetah-v3& 7030$\pm$772  &8259$\pm$649&8672$\pm$602&8109$\pm$1249&8339$\pm$730&12429$\pm$1048&11255$\pm$2027& 
12877$\pm$778&\bf 13386$\pm$694\\
 Ant-v3& 848$\pm$431 &5697$\pm$202&3906$\pm$766&4549$\pm$785&5361$\pm$644&4083$\pm$523&5631$\pm$335&\bf 6632$\pm$398 &5536$\pm$446\\
Hopper-v3&2150$\pm$578 &3216$\pm$561&2552$\pm$483&2285$\pm$430&3130$\pm$554&\bf 3467$\pm$253&806$\pm$490&3321$\pm$351&3175$\pm$489 \\
Walker2d-v3& 2045$\pm$587& 4008$\pm$485 &4419$\pm$259&4472$\pm$316&4159$\pm$457&4978$\pm$439 &3102$\pm$763&5247$\pm$427&\bf 5568$\pm$220\\

Humanoid-v3&297$\pm$90 &5003$\pm$255&4892$\pm$346&4844$\pm$406&5117$\pm$422&5385$\pm$252&3444$\pm$931&\bf 7359$\pm$573&5276$\pm$183\\
InvertedDouble-v2& 8766$\pm$580& 7293$\pm$2829&9293$\pm$82&9242$\pm$216&9088$\pm$542&4140$\pm$3750&8628$\pm$1439&\bf 9358$\pm$4&9344$\pm$20\\
InvertedPendulum-v2&954$\pm$90&995$\pm$9&991$\pm$16&985$\pm$28&985$\pm$29&546$\pm$404 &954$\pm$90&\bf 1000$\pm$0 &996$\pm$ 9\\
\bottomrule
\end{tabular}
\caption{Average return of the last 10 evaluation scores on MuJoCo environments over 10 random seeds, where $\pm$ captures a 95\% confidence interval. The maximum values of each row are bolded. ``InvertedDouble-v2" denotes the InvertedDoublePendulum-v2 task for brevity. 
}
\label{tab:mojoco_res}
\end{table*}

Fig.~\ref{fig:Over_results} shows that the proposed B2PD and B2PD (TD3) algorithms consistently outperform vanilla algorithms across multiple tasks, demonstrating significant performance gains.
As summarized in Table~\ref{tab:mojoco_res}, the proposed B2PD method achieves state-of-the-art (SOTA) performance on 6 out of 7 MuJoCo continuous control tasks. On average, B2PD improves the return by 22.2\% over the vanilla SAC baseline, while B2PD (TD3) surpasses the standard TD3 by 26.1\%, demonstrating consistent and significant performance gains across diverse continuous-control environments.

We further compare B2PD with representative exploration- and representation-based baselines. Although ICM, which relies on curiosity-driven intrinsic rewards, yields moderate improvements over SAC on HalfCheetah and Ant, its performance is inherently limited by sensitivity to stochastic dynamics and prediction noise, which may encourage exploration that is poorly aligned with task rewards.
The ALH method constructs actions by combining the previous-step representation with the current state embedding, which often produces highly correlated consecutive actions. This temporal redundancy restricts action diversity and leads to suboptimal exploration and performance degradation in several tasks.
We also observe that NNAC exhibits noticeable performance gains over SAC on HalfCheetah, Hopper, and Walker2d, while suffering performance drops on simpler tasks such as InvertedPendulum and InvertedDoublePendulum. This phenomenon suggests that, in simple environments where policies converge rapidly, effective policy guidance critically depends on the availability of reliable high-value anchors, in contrast to complex tasks where sustained exploration remains essential.

Additionally, we observe that the BAC algorithm exhibits noticeable performance degradation on the Hopper, Walker2d, and Humanoid tasks compared to SAC. We conjecture that excessive reliance on historical policies may bias action selection toward exploitation, thereby impairing effective exploration in environments with complex and highly dynamic behaviors. In contrast, B2PD consistently enhances performance by generating diverse, high-value policy priors that explicitly anchor policy updates within high-quality behavior support.

\textbf{Main results on state and pixel-based DMControl environments.}
We evaluate the proposed B2PD algorithm built upon SAC and TD3 on the state-based DMControl benchmark. To further assess its generality in visual-input scenarios, we integrate B2PD with RAD and DrQ-v2, and report the corresponding learning curves in Fig.~\ref{fig:Over_results_dmc} and final performance comparisons in Table~\ref{tab:mojoco_res_dmc}.
Experimental results demonstrate that the proposed B2PD algorithm significantly accelerates policy convergence. Across all four tasks, both B2PD and B2PD (TD3) consistently outperform the vanilla SAC and TD3 algorithms, respectively, further demonstrating the effectiveness and generality of our method.

\begin{figure*}[!htbp]
\begin{center}
\centerline{\includegraphics[width=\textwidth]{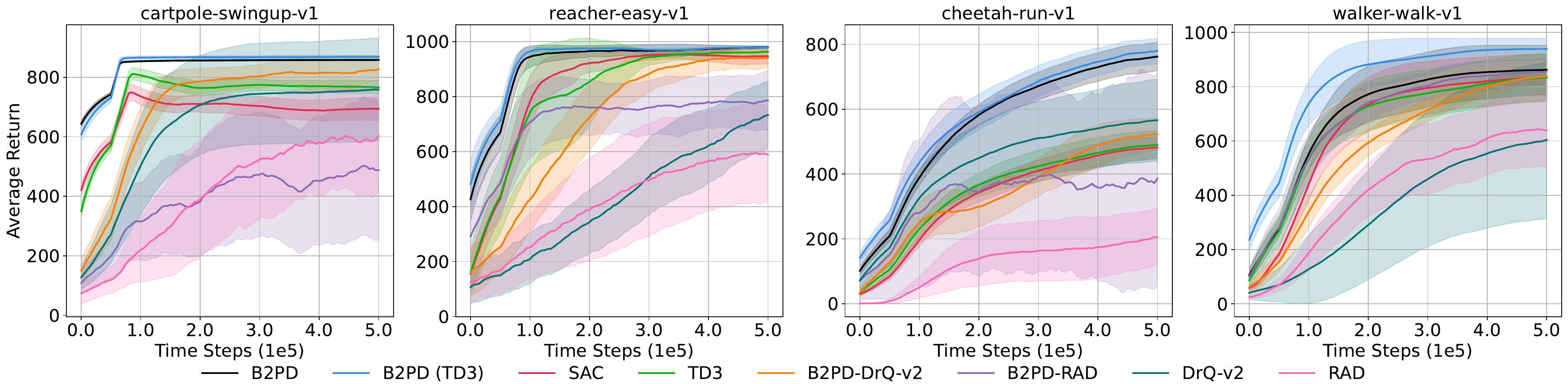}}
\caption{Learning curves on four DMControl tasks. The shaded area captures a 95\% confidence interval around the average performance over 10 trials. Curves are smoothed uniformly for visual clarity.}
\label{fig:Over_results_dmc}
\end{center}
\end{figure*}

We further observe that pixel-based methods such as RAD and DrQ-v2 suffer substantial performance degradation on several control tasks, including reacher-easy, cheetah-run, and walker-walk, compared with state-based methods. This degradation highlights the optimization challenges imposed by high-dimensional visual observations.
In contrast, when integrated with B2PD, DrQ-v2 achieves substantial performance gains: it surpasses state-based methods on cartpole-swingup and walker-walk, and attains competitive performance on reacher-easy. These results demonstrate that B2PD effectively mitigates the sample inefficiency of pixel-based reinforcement learning and exhibits generality across both state- and pixel-based observation modalities.
 \begin{table*}[!htbp]
\centering
\setlength\tabcolsep{3.5pt}
\begin{tabular}{lcccc|cccc}
\toprule
Settings      &\multicolumn{4}{c}{State} & \multicolumn{4}{c}{Pixel}\\  
\midrule
Tasks      &TD3&SAC&B2PD (TD3)&B2PD & RAD&DrQ-v2&B2PD-RAD &B2PD-DrQ-v2\\  
\midrule
cartpole-swingup-v1& 766$\pm$23  &695$\pm$39&\bf 869$\pm$5& 858$\pm$3&595$\pm$191& 
758$\pm$173&487$\pm$234&\bf 828$\pm$36\\
reacher-easy-v1&963$\pm$13 &946$\pm$29&979$\pm$6&\bf 980$\pm$5& 579$\pm$173&738$\pm$121&777$\pm$110&\bf 938$\pm$38 \\
cheetah-run-v1& 490$\pm$43& 482$\pm$32 &\bf 782$\pm$39&764$\pm$43&202$\pm$88 &\bf 566$\pm$127&400$\pm$327& 525$\pm$47\\
walker-walk-v1& 834$\pm$89& 835$\pm$85 &\bf 939$\pm$39&863$\pm$90&631$\pm$136 &605$\pm$290&\bf856$\pm$47& 840$\pm$29\\

\bottomrule
\end{tabular}
\caption{Average return of the last 10 evaluation scores on DMControl environments over 10 random seeds, where $\pm$ captures a 95\% confidence interval. The maximum values of each row are bolded.}
\label{tab:mojoco_res_dmc}
  
\end{table*}

\textbf{Main results on PyBullet environments.}
We evaluate B2PD by combining it with TD3 and SAC, and compare it against ICM, where intrinsic motivation is implemented following \cite{pathak2017curiosity}. We conducted experiments on the PyBullet benchmark suite, training each algorithm for 2 million timesteps using 10 random seeds. The learning curves are summarized in Fig.~\ref{fig:Over_results_pybullet_all}, and the final performance is reported in Table~\ref{tab:mojoco_res_pybullet}.
 
\begin{figure*}[!htbp]
\begin{center}
\centerline{\includegraphics[width=\textwidth]{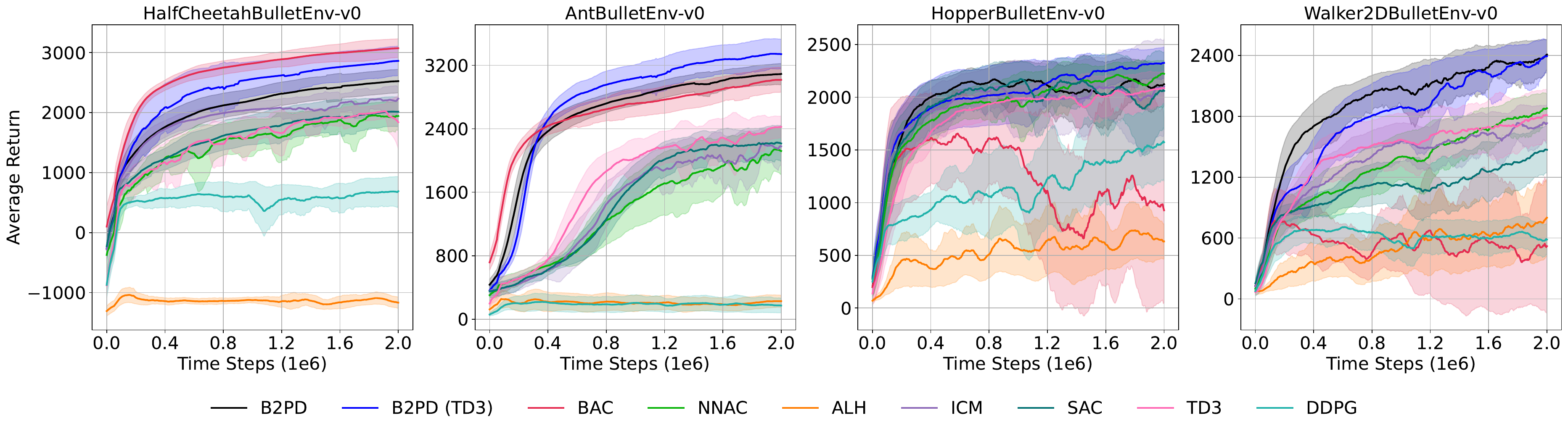}}
\caption{Learning curves on four PyBullet continuous control tasks. The shaded area captures a 95\% confidence interval around the average performance over 10 trials. Curves are smoothed uniformly for visual clarity.}
\label{fig:Over_results_pybullet_all}
\end{center}
\end{figure*}

  \begin{table*}[!htbp]
\centering
\setlength\tabcolsep{3.5pt}
 
\begin{tabular}{lccccccc|cc} 
\toprule
Tasks      &DDPG&TD3&SAC&ICM&ALH&NNAC&BAC&B2PD (TD3)& B2PD\\  \midrule
HalfCheetahBulletEnv-v0& 693$\pm$235  &1820$\pm$482& 2015$\pm$246&2231$\pm$188&-1174$\pm$94&1948$\pm$235&\bf 3075$\pm$157& 2863$\pm$247&2524$\pm$195\\
AntBulletEnv-v0&170$\pm$131 &2423$\pm$140&2216$\pm$230&2176$\pm$249&224$\pm$75&2131$\pm$289&3018$\pm$155&\bf 3338$\pm$190& 3091$\pm$131 \\
HopperBulletEnv-v0& 1586$\pm$365& 2082$\pm$157 & 2075$\pm$374&2106$\pm$450&640$\pm$174&2218$\pm$127&916$\pm$868&\bf 2326$\pm$153&2123$\pm$214 \\
Walker2DBulletEnv-v0& 583$\pm$153& 1819$\pm$247 & 1480$\pm$228&1736$\pm$174&800$\pm$386&1874$\pm$146&532$\pm$674&2390$\pm$166&\bf 2411$\pm$149\\

\bottomrule
\end{tabular}
\caption{Average return of the last 10 evaluation scores on Pybullet environments over 10 random seeds, where $\pm$ captures a 95\% confidence interval. The maximum values of each row are bolded.}
\label{tab:mojoco_res_pybullet}
\end{table*}

Across all four PyBullet continuous control tasks, both B2PD and B2PD (TD3), instantiated on top of SAC and TD3 respectively, consistently outperform their vanilla counterparts. These substantial and stable performance gains indicate that B2PD effectively improves the practical sample efficiency and robustness of online RL algorithms under noisy dynamics.

In comparison, ICM achieves only moderate improvements on HalfCheetahBulletEnv and Walker2DBulletEnv, while providing marginal benefits over vanilla SAC on the remaining tasks. We conjecture that the intrinsic reward signal in ICM is susceptible to stochastic dynamics noise in PyBullet environments, which introduces non-stationary learning targets and destabilizes policy optimization. Consequently, curiosity-driven exploration does not always translate into task-relevant performance gains.
We further observe that ALH, as an extension of TD3, underperforms the vanilla TD3 across all four PyBullet tasks. Unlike MuJoCo, PyBullet environments exhibit stronger dynamics stochasticity. By augmenting the actor’s decision with the previous timestep’s state representation, ALH amplifies noise propagation across timesteps, leading to degraded action quality and inferior performance.

Additionally, the BAC algorithm demonstrates competitive performance on HalfCheetahBulletEnv and AntBulletEnv, but suffers noticeable degradation on HopperBulletEnv and Walker2DBulletEnv, revealing pronounced performance instability across environments. We attribute this behavior to BAC’s heavy reliance on early-stage experiences, which biases policy updates toward premature exploitation and suppresses optimistic exploration. This limitation becomes particularly detrimental in tasks with sensitive dynamics or narrow stability regions, and is consistent with BAC’s observed performance degradation on the Hopper and Humanoid tasks in the MuJoCo benchmark.

\subsection{Comparison of Distillation Loss for B2PD}

To constrain the deviation between the current actor’s decisions and the expert prior, we draw inspiration from knowledge distillation~\cite{sun2024logit} by introducing commonly used KL divergence~\cite{hinton2015distilling} and L2 distance metrics. The experimental results are shown in Fig.~\ref{fig:Over_results_mujoco_KL}.
\begin{figure*}[!htbp]
\begin{center}
\centerline{\includegraphics[width=0.98\textwidth]{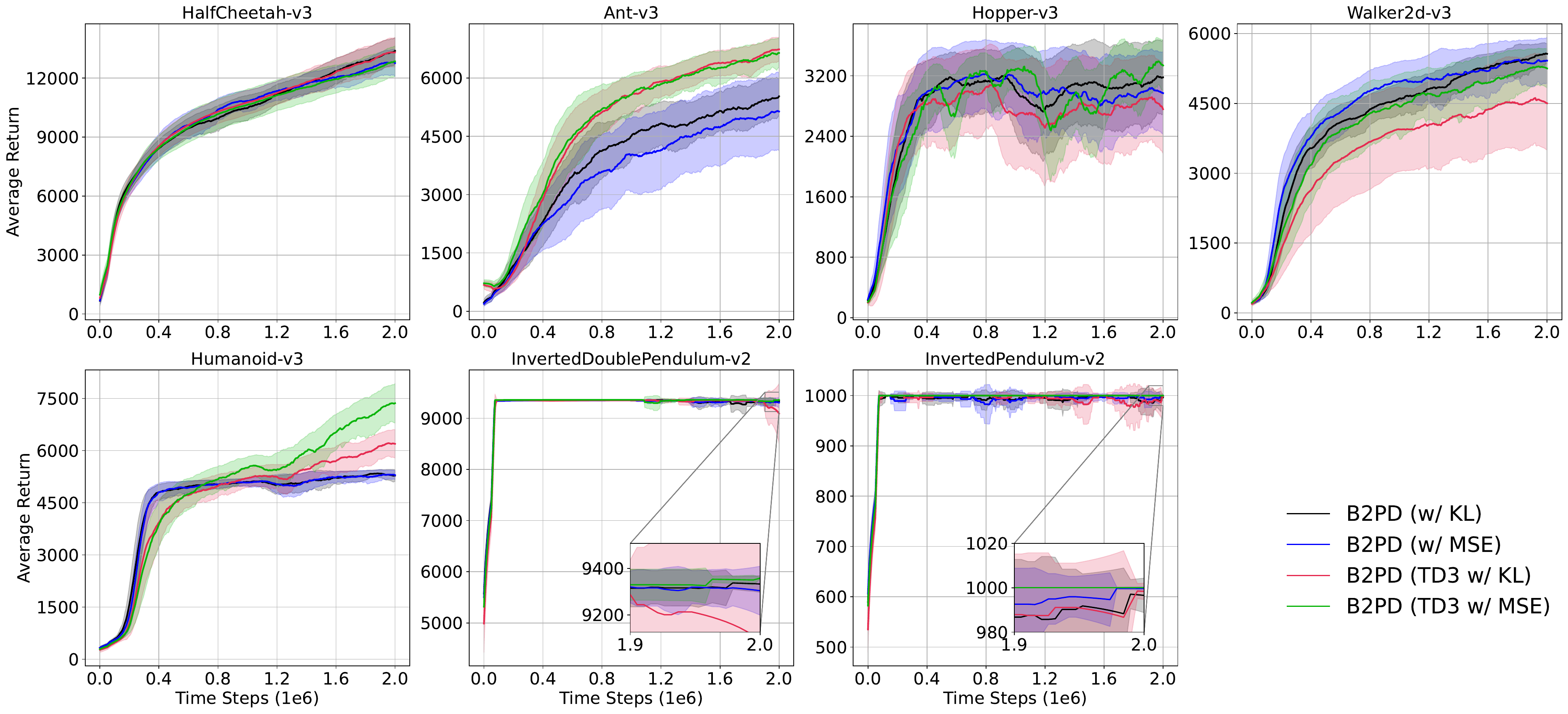}}
\caption{Learning curves on seven MuJoCo continuous control tasks. The shaded area captures a 95\% confidence interval around the average performance over 10 trials. Curves are smoothed uniformly for visual clarity.}
\label{fig:Over_results_mujoco_KL}
\end{center}
\end{figure*}

It can be observed that using KL divergence as the optimization objective leads to a significant performance gain in the B2PD framework, while MSE loss achieves superior results under the B2PD (TD3) setting. We attribute this difference to the distinct nature of the policy structures: KL divergence loss provides greater stability for stochastic policies parameterized by Gaussian distributions, whereas MSE loss offers more direct and explicit supervision for deterministic policy updates. This finding highlights the importance of selecting distillation objectives that are aligned with the underlying policy structure, thereby enabling more efficient knowledge transfer and policy convergence.

\subsection{Comparison with Additional Noise Scheduling Methods}
Since noise injection has been shown to stabilize Q-value estimation, we systematically evaluate the effectiveness of different noise-scheduling strategies in both the B2PD and SAC frameworks. Specifically, we compare B2PD with SDA Noise, and several alternative noise-scheduling approaches, including standard Gaussian noise~\cite{fujimoto2018addressing}, linearly decayed Gaussian noise~\cite{hollenstein2022action}.
To further isolate and understand the role of noise-scheduling mechanisms within SAC, we additionally apply the same set of noise-scheduling strategies to the vanilla SAC algorithm. All configurations are trained for 2M environment steps using 10 random seeds. The detailed results are summarized in Fig.~\ref{fig:Over_results_ant_noise}.

\begin{figure}[!ht]
  \centering
  \includegraphics[width=0.5\textwidth]{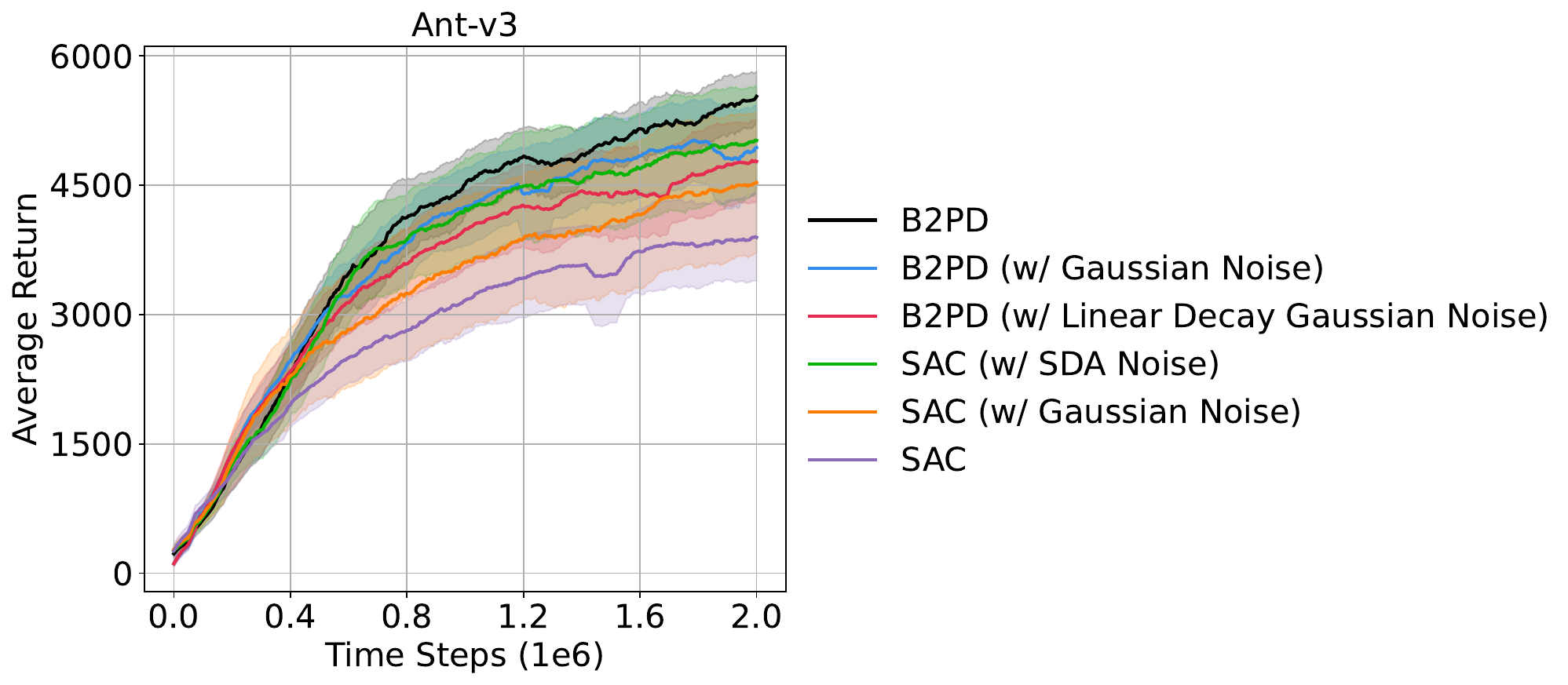}
  \caption{Comparison of the impact of different policy noises on the performance of the B2PD algorithm. The shaded area represents the 95\% confidence interval over the average performance of 10 trials. Curves are uniformly smoothed for visual clarity.}
  \label{fig:Over_results_ant_noise}
\end{figure}

As shown in Fig.~\ref{fig:Over_results_ant_noise}, B2PD (w/ Linear Decay Gaussian Noise) performs worse than the standard B2PD configuration. This degradation suggests that the progressive reduction of noise magnitude during training, imposed by the linear decay schedule, may weaken the smoothing effect that is critical for stable Q-value estimation. Furthermore, both B2PD and SAC (w/ SDA Noise) respectively outperform B2PD (w/ Gaussian Noise) and SAC (w/ Gaussian Noise), highlighting the effectiveness of SDA Noise modulation.

\subsection{Does B2PD Improve Exploration Efficiency?}
To evaluate whether B2PD improves exploration efficiency, we compare B2PD with SAC by examining both the task performance and the policy entropy throughout training on the Ant benchmark, as shown in Fig.~\ref{ablation_performace_entropy}.
\begin{figure}[!htbp]
\centering
\includegraphics[width=0.5\textwidth]{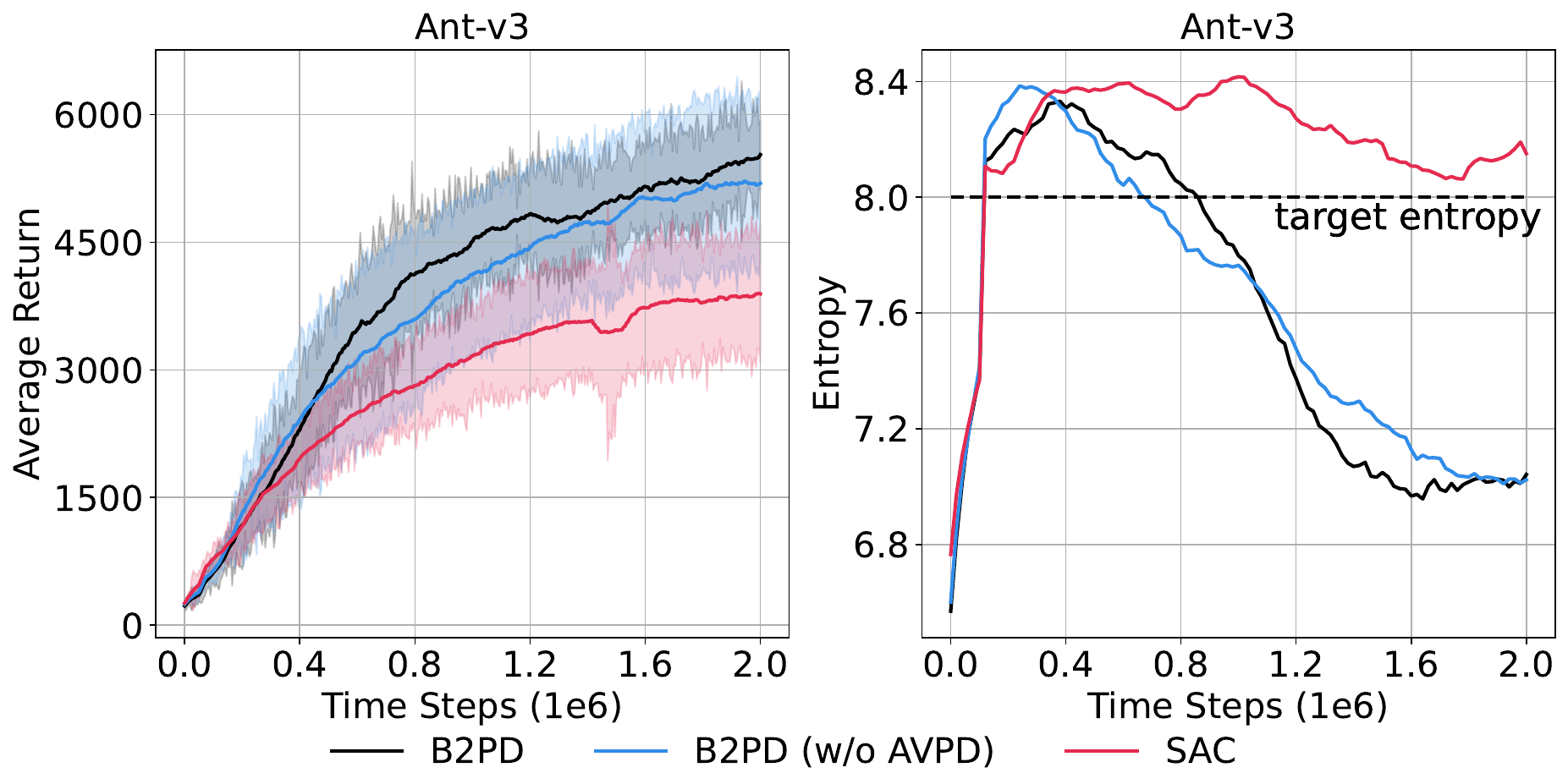}
\caption{Comparison of evaluation performance and policy entropy distribution over 2M training timesteps. Curves are uniformly smoothed for clarity.}
\label{ablation_performace_entropy}
 
\end{figure}

As shown in Fig.~\ref{ablation_performace_entropy}, in the early stages of training, both B2PD and B2PD (w/o AVPD) maintain high-entropy action distributions comparable to vanilla SAC, thereby preserving sufficient exploration. As training progresses, both B2PD and B2PD (w/o AVPD) rapidly converge to lower-entropy policies, highlighting the effectiveness of behavior-prior distillation in stabilizing policy optimization. These observations suggest that behavior priors effectively guide policy updates while reducing inefficient exploration.

Although B2PD (w/o AVPD) and B2PD exhibit a similar exploration pattern, B2PD (w/o AVPD) underperforms the full B2PD framework in terms of final performance. This result indicates that action-value prior distillation enhances the quality of the generated policy support set by producing more informative action priors. Consequently, the actor is guided toward higher-quality policy updates, leading to more effective policy optimization and superior overall performance.

\subsection{Mitigating Value Estimation Error with B2PD}
To evaluate the effectiveness of B2PD in mitigating value estimation errors, we experimentally compare the evolution of the discrepancy between the predicted Q-values and the true discounted returns during training. Specifically, we contrast B2PD with the SAC, and report the corresponding value estimation errors as shown in Fig.~\ref{ablation_overestimation}.

\begin{figure}[!htbp]
\centering
\includegraphics[width=0.5\textwidth]{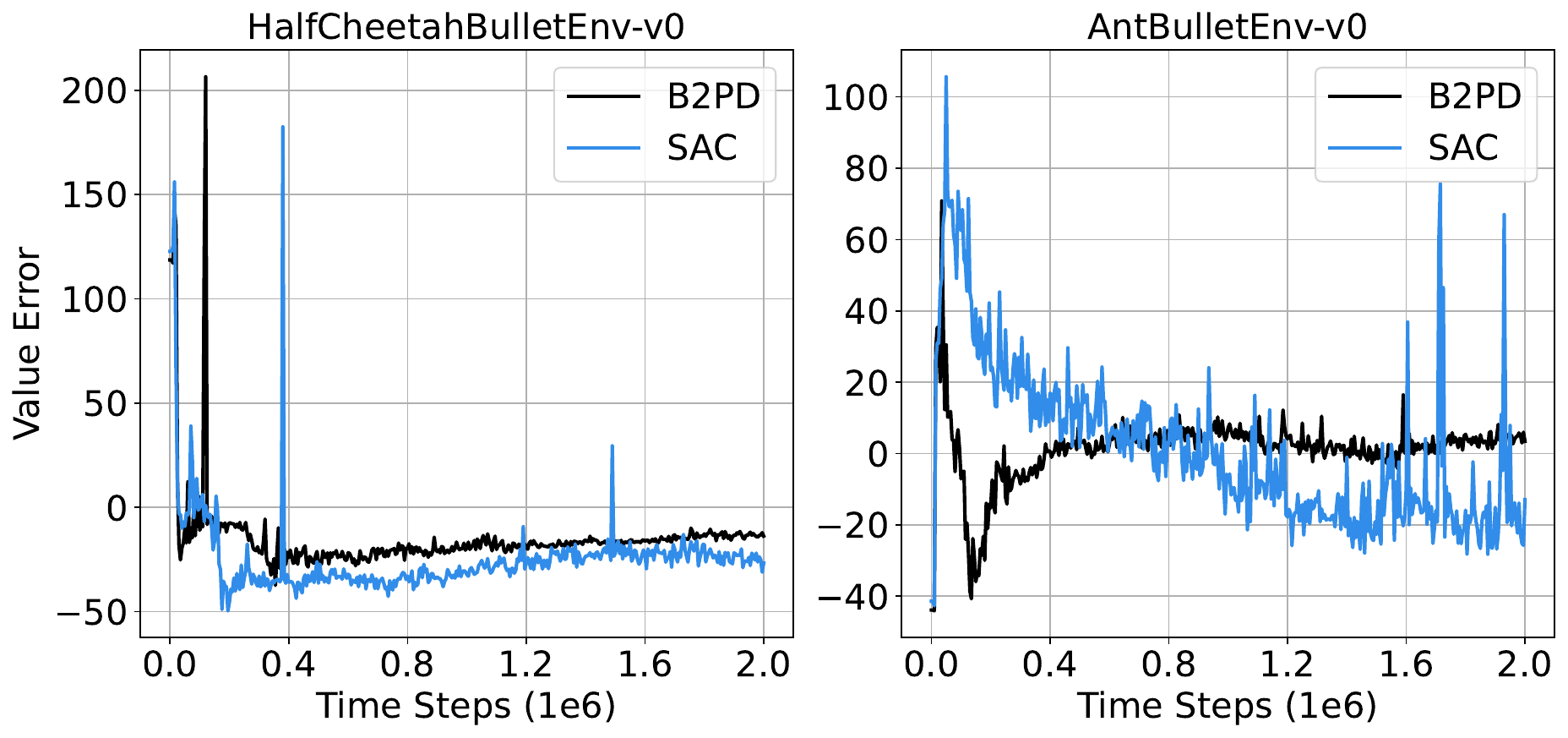}
\caption{The overestimation bias of value estimates in the B2PD and SAC algorithms is measured over 2 million time steps.}
\label{ablation_overestimation}
 
\end{figure}

As shown in Fig.~\ref{ablation_overestimation}, B2PD exhibits faster value convergence on the two evaluated tasks. Specifically, compared to SAC, the Q-values predicted by B2PD are closer to the true discounted returns and exhibit lower variance. In particular, on the AntBulletEnv task, the value estimation error of B2PD approaches zero in the later stages of training, whereas SAC displays more oscillatory errors. We attribute this to the fact that SAC still suffers from blind exploration in the later training stages, leading to non-stationary value estimates. In contrast, B2PD leverages an expert support-set policy as an explicit constraint, providing more meaningful gradients for policy updates and reducing ineffective exploration.

\subsection{Hyperparameter Study}
B2PD introduces three hyperparameters: the action-value prior distillation coefficient $\xi$, the number of sampled behavioral priors $H$, and the behavior prior distillation coefficient $\eta$. 
To demonstrate the parameter sensitivity of B2PD, we trained for 1 million timesteps on PyBullet tasks over 10 different random seeds.

\begin{figure}[!htbp]
\begin{center}
\centerline{\includegraphics[width=0.5\textwidth]{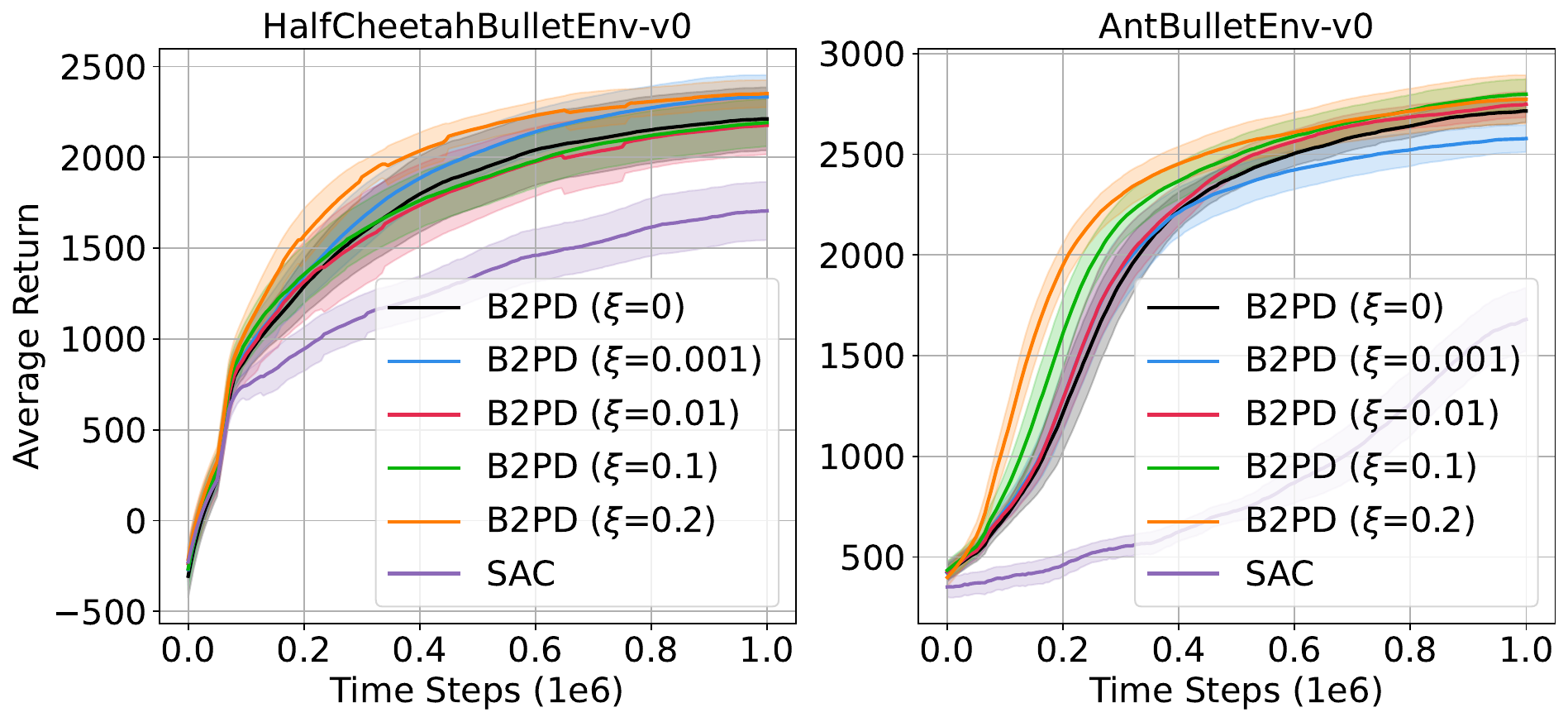}}
\caption{Sensitivity analysis with respect to parameter $\xi$. The shaded area captures a 95\% confidence interval around the average performance over 10 trials. Curves are smoothed uniformly for visual clarity.}
\label{fig:Over_results_pybullet_xi}
\end{center}
\end{figure}
\textbf{Action-value prior distillation coefficient $\xi$.}
The action-value prior distillation coefficient $\xi$ is a key hyperparameter in B2PD, as it directly controls the trade-off between behavior reconstruction and high-value action generation in the CVAE training objective. 
When $\xi = 0$, the proposed method degenerates to B2PD without action-value prior distillation (denoted as B2PD w/o AVPD).
We evaluate the effect of $\xi$ on four PyBullet control tasks using the SAC backbone, with $\xi \in \{0, 0.001, 0.01, 0.1, 0.2\}$. 
The results are summarized in Fig.~\ref{fig:Over_results_pybullet_xi}.

From Fig.~\ref{fig:Over_results_pybullet_xi}, we make the following observations. 
First, both B2PD and its degenerate variant B2PD ($\xi = 0$) consistently outperform the vanilla SAC baseline, indicating that incorporating policy priors into the update process already provides a clear performance advantage.
Second, introducing a positive distillation weight $\xi$ further improves performance across all tasks, demonstrating the effectiveness of the Q-guided loss term $\mathcal{H}_{\text{Q}}(\omega)$ in biasing the generative policy toward higher-value actions.
This result highlights the importance of explicitly injecting value-aware guidance during policy optimization.

Moreover, compared to B2PD ($\xi = 0$), B2PD exhibits robustness to the choice of $\xi$ in the range $[0.01, 0.1]$ across HalfCheetahBulletEnv and AntBulletEnv environments, with all settings yielding substantial performance gains. In contrast, when $\xi = 0.2$, there is a slight performance drop compared to $\xi = 0.1$. We attribute this to the fact that a larger Q-value-guided policy support set learning may generate actions that deviate more from the current policy, potentially leading to training instability.

\textbf{Number of sampled behavioral priors $H$.}
When the behavioral prior policy $G$ is known, $H$ should be set to a sufficiently large value to optimally sample actions with the highest Q-value. However, in our practical setting, we use a CVAE to approximate the support set of the behavior policy prior and sample high-value actions from it. Here, $H$ takes on the role of balancing optimistic optimization and constrained exploration.
To explore the influence of $H$, we conducted experiments on four tasks with a fixed hyperparameter $\xi$=0.1.
First, we evaluated the impact of varying $H$ on the HopperBulletEnv and Walker2DBulletEnv tasks, using $H \in \{1, 10, 20\}$ with 10 random seeds per setting.
In these settings, both Q-values and generative policy priors are leveraged to guide the policy update process.
The results demonstrate that expert prior effectively enhances performance, as illustrated in Fig.~\ref{fig:Over_results_pybullet_H}.
\begin{figure}[!htbp]
\begin{center}
\centerline{\includegraphics[width=0.5\textwidth]{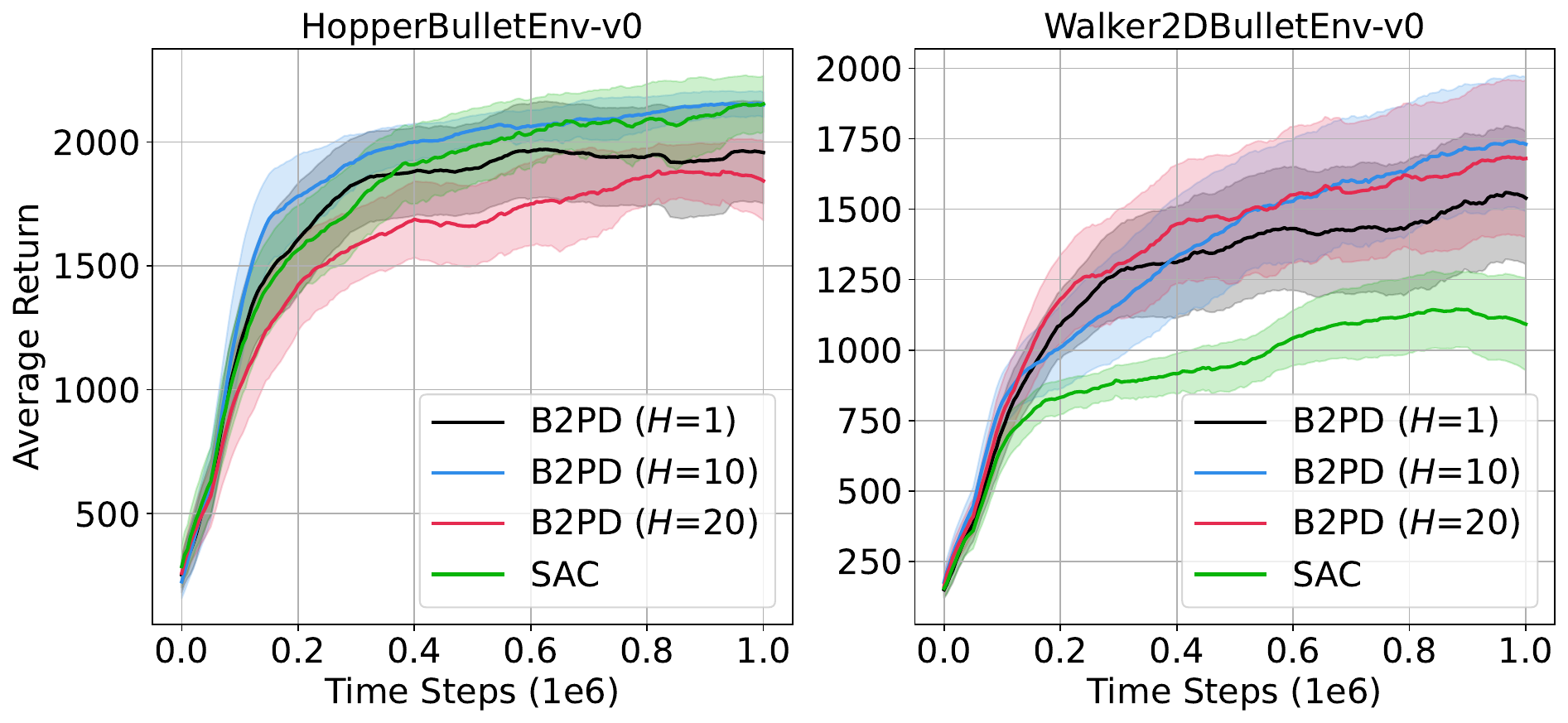}}
\caption{Sensitivity analysis with respect to parameter $H$. The shaded area captures a 95\% confidence interval around the average performance over 10 trials. Curves are smoothed uniformly for visual clarity.}
\label{fig:Over_results_pybullet_H}
\end{center}
\end{figure}

We observe that B2PD significantly improves the performance of the vanilla SAC algorithm across varying numbers of policy priors $H$. This indicates that incorporating expert policy priors as an update anchor can effectively mitigate the issue of suboptimal exploration induced by Q-guidance. However, as $H$ increases, the likelihood of sampling previously seen low-quality actions also increases, which can introduce instability during training. Based on this empirical observation, we set $H = 10$ in all experiments to balance performance gains and training stability.

\textbf{Behavior prior distillation coefficient $\eta$.}
The behavior prior distillation coefficient $\eta$ controls the trade-off between behavior prior policy distillation and the maximum-entropy policy optimization objective.
When $\eta = 0$, the proposed B2PD framework degenerates into a variant without the behavior prior policy distillation module, denoted as B2PD (w/o BPPD). In this case, a modest performance improvement can still be observed in the AntBulletEnv task, indicating that the remaining components of B2PD contribute to policy optimization.
\begin{figure}[!htbp]
\begin{center}
\centerline{\includegraphics[width=0.5\textwidth]{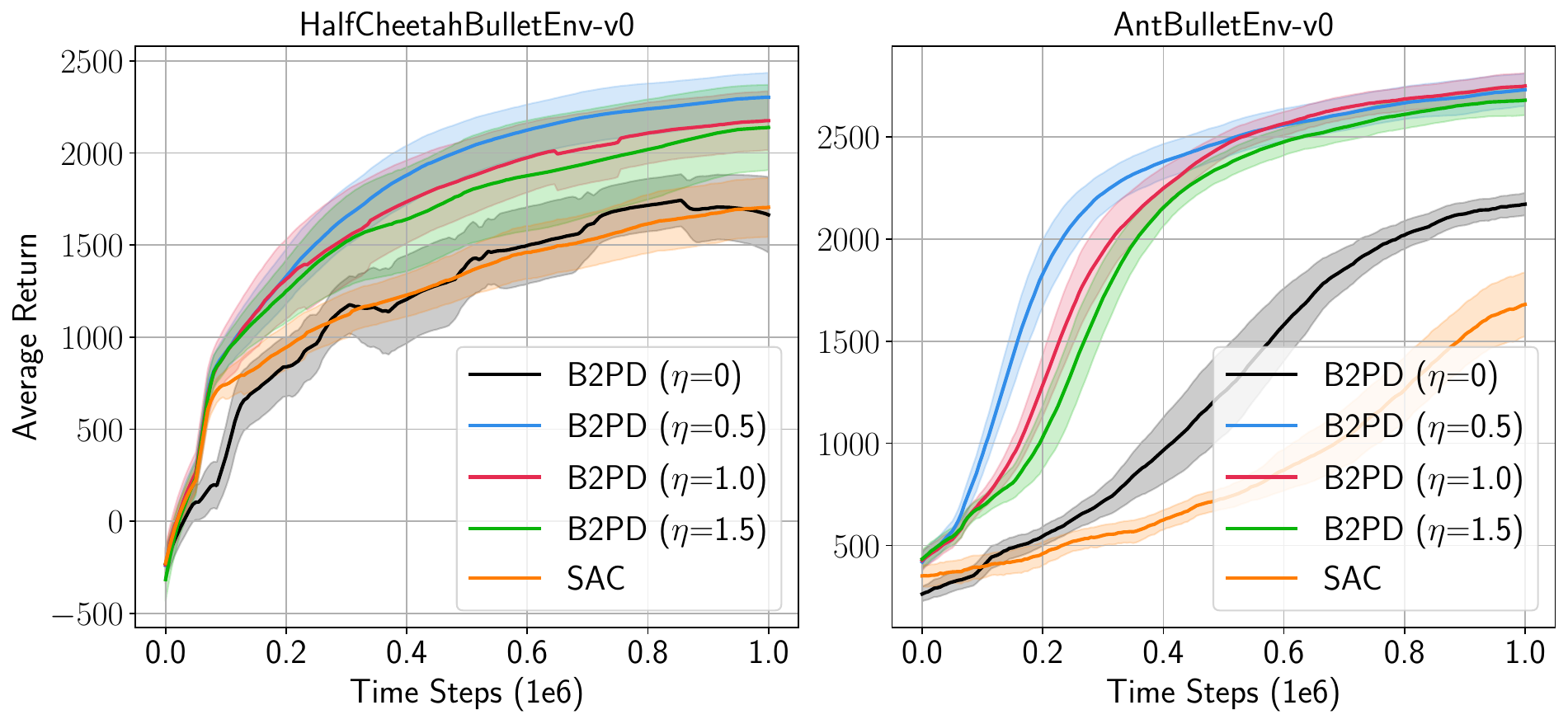}}
\caption{Sensitivity analysis with respect to parameter $\eta$. The shaded area captures a 95\% confidence interval around the average performance over 10 trials. Curves are smoothed uniformly for visual clarity.}
\label{fig:Over_results_pybullet_eta}
\end{center}
\end{figure}

As shown in Fig.~\ref{fig:Over_results_pybullet_eta}, B2PD achieves its best performance on AntBulletEnv when $\eta = 1.0$. In contrast, for the HalfCheetahBulletEnv task, $\eta = 0.5$ yields the optimal performance, suggesting that the optimal choice of $\eta$ is task-dependent and requires careful tuning to fully exploit the benefit of behavior prior distillation.

\textbf{Can the B2PD algorithm perform policy distillation under noisy priors?}
In online RL, if the behavioral priors concentrate in low-value regions, they exacerbate the sample-efficiency problem and undermine the effectiveness of policy distillation. To assess the robustness of B2PD under such conditions, we uniformly inject state noise during the inference phase of the prior, defined as $s_t^{\star} = s_t + \mathcal{N}(0, \epsilon_s^2)$, where $\epsilon_s \in \{0.1, 0.3, 1.0\}$.
We construct two versions of behavioral prior distributions. The first is inferred from noisy states, where Gaussian noise with scale $\epsilon_s$ is injected into the state inputs. The second is inferred from clean states without corruption. Given these priors, we use the Q-network to select the expert policy anchor, formalized as
\begin{equation}
    {\tilde {a}} = {\arg\mathop{\max}\limits_{a_h}} \left(Q_{\theta^{'}}(s_t,a_h)\right), a_h \in \Pi_{{G_\omega}(\cdot \mid \{s_t,s_t^{\star}\})}.
\end{equation}

To further assess whether B2PD can filter out high-quality priors even when the action support itself becomes unreliable, we inject action noise $\epsilon_a$ into the action support set. This noise is applied to each action proposal before anchor selection. The noisy-anchor filtering process is then expressed as
\begin{equation}
\tilde{a}
= \arg\max_{a \in \{a_h,\, a_h^\star\}}
Q_{\theta'}(s_t, a), 
a_h \sim \Pi_{G_\omega}(\cdot | s_t), 
a_h^\star = a_h + \epsilon_a.
\end{equation}

We conduct experiments on the HalfCheetahBulletEnv and AntBulletEnv tasks for 1M timesteps, and report results averaged over 10 random seeds. The corresponding performance curves are shown in Fig.~\ref{ab_noise_rl}.
\begin{figure}[!htbp]
    \centering
    \includegraphics[width=0.5\textwidth]{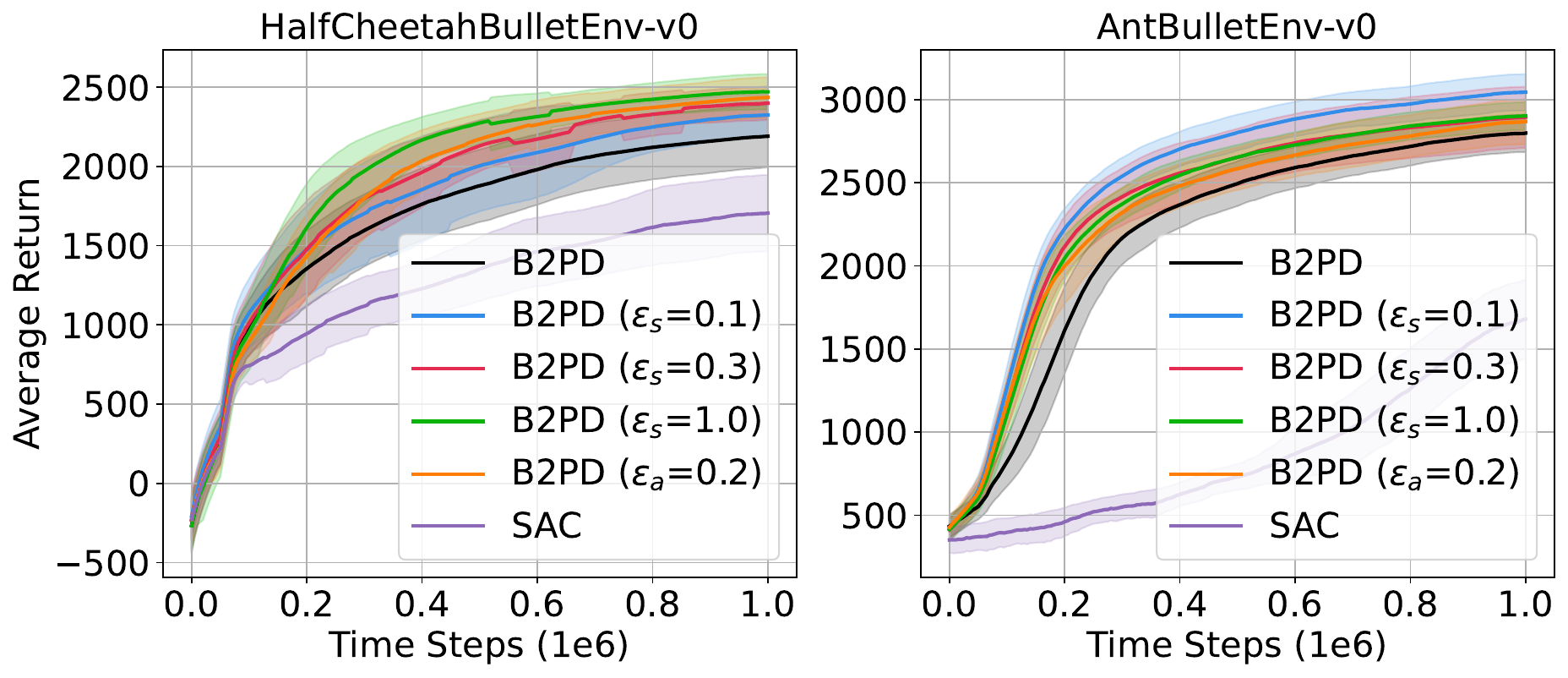}
    \caption{Average return over 10 seeds with different levels of conditional state-noise and action-noise applied to the behavioral prior. The shaded area represents the 95\% confidence interval over the average performance of 10 trials. Curves are uniformly smoothed for visual clarity.}
    \label{ab_noise_rl}
\end{figure}

Fig.~\ref{ab_noise_rl} demonstrates that B2PD consistently benefits from Gaussian state noise injected at three different intensity levels, and action noise further contributes to performance improvements. This improvement arises from two complementary effects: first, the proposed SDA Noise mechanism stabilizes Q-value learning, which helps to extract high-quality policy anchors; second, the injected noise increases the diversity of the inferred action priors, providing the actor with a richer set of high-value actions.
Notably, these noise-conditioned priors differ fundamentally from actions sampled from clean states: the former correspond to neighboring strategies of the current policy, capturing its potential future trends, whereas the latter may include more out-of-distribution actions. As a result, priors inferred from noisy states exhibit stronger robustness, allowing B2PD to achieve higher sample efficiency and faster policy convergence in purely online RL settings. These results further underscore the promise of B2PD for complex, real-world scenarios.

\begin{table*}[!htbp]
\centering
\setlength\tabcolsep{3.5pt}
\caption{Wall-clock training time of each algorithm on the HalfCheetah-v3 task, measured over 2M environment steps.}

\begin{tabular}{lccccccc|cc}
\toprule
Methods        &DDPG&TD3&SAC&ICM&ALH&NNAC&BAC&B2PD (TD3)&B2PD    \\  
 \midrule
Runtime&8.1h& 5.4h &6.3h&7.4h&6.8h&11.5h&18.3h&9.5h&10.6h\\
\bottomrule
\end{tabular}
\label{tab:consumer_time}
\end{table*}

\subsection{Runtime Analysis}
We evaluated the computational cost of the proposed B2PD method in a Linux environment equipped with a 56-core Intel Xeon(R) 6133 CPU and a single NVIDIA RTX 3090 GPU, and compared it against DDPG, TD3, SAC, ICM, ALH, NNAC, and BAC. The runtime statistics are presented in Table~\ref{tab:consumer_time}.

As reported in Table~\ref{tab:consumer_time}, B2PD incurs higher training time than vanilla baselines, with both B2PD and B2PD (TD3) exhibiting noticeable computational overhead.
This additional cost primarily arises from the training and sampling procedures of the CVAE module, which introduces extra forward and backward passes during policy optimization. 
Furthermore, the larger runtime increase of B2PD (TD3) compared to TD3 can be attributed to its aggressive policy update strategy, where the policy is updated at every iteration ($d=1$), resulting in more frequent actor–critic optimization steps.

We also observe that BAC demonstrates substantial runtime overhead, which is likely caused by its complex multi-branch architecture and coupled optimization pipeline.
Overall, while B2PD introduces additional computational cost, this overhead reflects a deliberate trade-off for improved sample efficiency and policy convergence, as validated by the performance gains reported in previous sections.

\section{Discussion}
\label{sec:limitations}
The empirical results show that B2PD differs fundamentally from prior online policy-prior guidance methods by using a generative expert policy prior as an update anchor to guide policy exploitation.
This mechanism accelerates policy convergence and improves overall performance, highlighting the potential of generative priors as a general tool for addressing sample inefficiency in online RL.
 
\textbf{Limitations.}
Despite demonstrating significant performance improvements in experiments, B2PD still exhibits certain notable limitations. Specifically, as discussed in Section~\ref{sec:toy_exper}, B2PD achieves rapid convergence in toy experiments with dense rewards; however, since the action value prior heavily relies on Q-value estimation, further investigation is needed to obtain high-value priors in sparse reward tasks.

\section{Conclusions and Future Work}
\label{sec:conclusion_future_works}
 
This work investigates the potential benefits of incorporating policy priors into the policy learning process. Building on this insight, we theoretically show that employing a sufficiently expressive distribution approximator can provide policy prior guidance to a Gaussian actor network during policy updates. Motivated by this observation, we propose B2PD to address the challenge of inefficient exploration.
Through extensive empirical evaluation, we demonstrate that, under fixed hyperparameter settings, B2PD improves algorithmic performance on both state-based and pixel-based tasks. Qualitative analyses in toy environments further show that B2PD effectively suppresses exploration in low-value regions by distilling high-value policy priors into the agent. 

In future work, we plan to focus on developing effective methods for learning meta-knowledge priors and investigating whether the learned prior policies can be transferred to other tasks. Further exploration of these questions would provide valuable insights into generalizable policy learning.


\bibliography{cas-refs}
\bibliographystyle{IEEEtran}

 \newpage
\clearpage

\numberwithin{equation}{section}
\numberwithin{figure}{section}
\numberwithin{table}{section}

\section{Theoretical Analyses}
\label{thero_ana}
\subsection{Q-Value Smoothness}
\label{thero_ana_b2d}
 \begin{proof}[Noise-Regularized Q-Value Smoothness]
\label{noise_smooth}
We assume that the action-value function $Q(s_t, a_t)$ is Lipschitz continuous concerning the action variable $a_t$ for any fixed state $s_t \in \mathcal{S}$. That is, there exists a constant $K_Q > 0$ such that for any $a_1, a_2 \in \mathcal{A}$,
\begin{equation}
\label{eq:lipschitz_Q}
||Q(s_t, a_1) - Q(s_t, a_2)|| \leq K_Q \|a_1 - a_2\|.
\end{equation}

This condition ensures that small perturbations in the action space induce proportionally bounded variations in the Q-value, thereby promoting local smoothness of the critic function. We derive a precise upper bound on the approximation error
\begin{equation}
\label{taylor_q_expansion}
\begin{aligned}
\E_\epsilon|Q(s_t, a_t &+ \epsilon) -Q(s_t, a_t )| \approx\E_\epsilon|Q(s_t, a_t) + \nabla_{a_t} Q(s_t, a_t)^\top \epsilon   \\\quad &+ \frac{1}{2} \epsilon^\top \nabla^2_a Q(s_t, \hat{a}) \epsilon-Q(s_t, a_t )|\\
&= \E_\epsilon|\nabla_{a_t} Q(s_t, a_t)^\top \epsilon+\frac{1}{2}  \epsilon^\top\nabla^2_a Q(s_t, \hat{a}) \epsilon|\\
&\leq \E_\epsilon|\nabla_{a_t} Q(s_t, a_t)^\top \epsilon|+\frac{1}{2} \E_\epsilon| \epsilon^\top\nabla^2_a Q(s_t, \hat{a}) \epsilon|,
\end{aligned}
\end{equation}
where $\hat{a} \sim \mathcal{U}(a_t, a_t + \epsilon)$ is an action sampled uniformly between $a_t$ and the noisy action $a_t + \epsilon$.
For any state $s_t$ and action $a_t$, the Q-function $Q(s_t, a_t)$ is twice differentiable with respect to $a_t$, and there exists two constants $C_1 > 0, C_2 > 0$ such that
\begin{equation}
\| \nabla_{a_t} Q(s_t,  a_t)\| \leq C_1,\| \nabla_{a_t}^2 Q(s_t, a_t) \| \leq C_2.
\end{equation}

We begin by analyzing the first term $\E_\epsilon|\nabla_{a_t} Q(s_t, a_t)^\top \epsilon|$. Since $\mathbb{E}[|X|] = 2 \int_0^\infty x \cdot \frac{1}{\sqrt{2\pi}\tau} \exp\left(-\frac{x^2}{2\tau^2}\right) \, dx = \tau \sqrt{\frac{2}{\pi}}, X \sim \mathcal{N}(0,\tau^2)$
by applying H\"{o}lder's inequality~\cite{folland1999real}, we obtain that
\begin{equation}
\begin{aligned}
   \E_\epsilon|\nabla_{a_t} Q(s_t, a_t)^\top \epsilon| &\leq \E_\epsilon(||\nabla_{a_t} Q(s_t, a_t) ||_\infty||\epsilon||_1)
   \\&\leq C_1 \E_\epsilon||\epsilon||_1 \\&\leq C_1\sqrt{\frac{2}{\pi}}\tau.
   \end{aligned}
\end{equation}

Next, we analyze the second term $\frac{1}{2} \E_\epsilon| \epsilon^\top\nabla^2_a Q(s_t, \hat{a}) \epsilon|$. According to the variance identity $\mathbb{E}[X^2] = \operatorname{Var}(X) + \left(\mathbb{E}[X]\right)^2$. In practice, we apply a symmetric clipping operation to the noise $\epsilon$, which ensures that $\text{Var}(clip(\epsilon)) \le \tau^2$ and $\mathbb{E}(\epsilon) = 0$. It follows that
\begin{equation}
\begin{aligned}
\E_\epsilon|\epsilon^\top\nabla^2_a Q(s_t, \hat{a}) \epsilon| &\leq ||\nabla^2_a Q(s_t, \hat{a})||\cdot \E_\epsilon||\epsilon^2|| \\&=   C_2(\text{Var}(\epsilon)+0)\le C_2\tau^2.
\end{aligned}
\end{equation}

Consequently, the expected deviation is bounded by $\E_\epsilon\left|Q(s_t, a_t + \epsilon) - Q(s_t, a_t)\right| \le C_1\sqrt{\frac{2}{\pi}}\tau+ \frac{1}{2}C_2\tau^2$.
\end{proof}
We observe that the upper bound of the approximation error increases as the policy variance grows. In the early stage of training, maintaining a low error bound is crucial for stabilizing Q-value estimation. In contrast, during the later stage, a larger noise variance is needed to preserve sufficient exploration and prevent the policy from converging to suboptimal solutions. Based on this observation, we propose an SDA Noise scheduling mechanism to enable adaptive noise modulation, thereby promoting more stable learning of Q-values.

\vfill

\end{document}